\documentclass{article}
\usepackage{PRIMEarxiv}
\errorcontextlines\maxdimen

\usepackage[utf8]{inputenc}
\usepackage[T1]{fontenc}
\usepackage{amsmath,amssymb,amsthm}
\usepackage{microtype}
\usepackage[dvipsnames,table]{xcolor}
\usepackage{dsfont}
\usepackage{url}
\usepackage[numbers,sort&compress]{natbib}

\usepackage{amsmath,amsfonts,bm}

\def\eqref#1{(\ref{#1})}

\def\1{\bm{1}}

\DeclareMathAlphabet{\mathsfit}{\encodingdefault}{\sfdefault}{m}{sl}
\SetMathAlphabet{\mathsfit}{bold}{\encodingdefault}{\sfdefault}{bx}{n}

\usepackage{amsmath,amsthm,amssymb,amsxtra,mathtools,bm}
\usepackage{array}
\usepackage{graphicx}
\usepackage{algorithm}
\usepackage{algpseudocode}
\usepackage{physics}
\usepackage{MnSymbol,wasysym}
\usepackage{tikzsymbols}
\usepackage{array,tabularx,multirow}
\usepackage{makecell}
\usepackage{array}

\usepackage{array}
\newcolumntype{P}[1]{>{\centering\arraybackslash}p{#1}}

\usepackage{verbatim}
\usepackage{enumerate}

\usepackage{color}
\definecolor{clemson-orange}{RGB}{234,106,32}
\definecolor{chicago-maroon}{RGB}{128,0,0}
\definecolor{cincinnati-red}{RGB}{190,0,0}
\definecolor{soft-cyan}{RGB}{68,85,90}
\definecolor{firebrick}{RGB}{178,34,34}
\definecolor{crimson}{RGB}{220,20,60}
\definecolor{cerrulean}{rgb}{0.165,0.322,0.745}
\definecolor{jaam}{rgb}{0.45,0.0,0.45}

\usepackage{hyperref}
\hypersetup{hidelinks}

\theoremstyle{plain}
\newtheorem{theorem}{Theorem}[section]
\newtheorem{lemma}[theorem]{Lemma}

\newtheorem{assumption}{Assumption}

\newif\ifsolutions \solutionstrue

\def\final{0}
\newcommand{\reviewer}[3]{
  \expandafter\newcommand\csname #1\endcsname[1]{
    \ifthenelse{\equal{\final}{1}} {
      \textcolor{#3}{}
    } {
      \textcolor{#3}{\begin{center} \textbf{#2} ##1 \end{center}}
    }
  }
}

\newcounter{qitem}

\renewcommand{\Pr}{\mathbb{P}}

\usepackage{booktabs}
\usepackage{multirow}
\usepackage{caption}
\usepackage[section]{placeins}
\usepackage{xurl}
\usepackage{cleveref}

\hypersetup{
  colorlinks, citecolor=blue, linkcolor=red, urlcolor=green,
  pdftitle={Training-Loss Guarantees for Muon with Finite-Step Newton-Schulz Orthogonalization},
  pdfauthor={Amartya Roy, Souvik Chakraborty}
}

\newcommand{\why}[1]{&&\text{\footnotesize\normalfont[#1]}}
\newtheorem*{theorem*}{Theorem}
\newtheorem*{lemma*}{Lemma}
\newtheorem{proposition}{Proposition}

\title{Training-Loss Guarantees for Muon with Finite-Step\\
Newton--Schulz Orthogonalization}

\author{
Amartya Roy \\
The School of Interdisciplinary Research \\
Indian Institute of Technology Delhi, India \\
Robert Bosch GmbH, India \\
\texttt{srz248670@iitd.ac.in} \\
\And
Souvik Chakraborty \\
Department of Applied Mechanics \\
Indian Institute of Technology Delhi, India \\
Yardi School of Artificial Intelligence (ScAI) \\
Indian Institute of Technology Delhi, India \\
\texttt{souvik@am.iitd.ac.in}
}

\begin{document}
\maketitle

\begin{abstract}
Existing convergence analyses of Muon either assume exact orthogonalization or analyze classical Newton--Schulz polynomials, and guarantee only stationarity, so it is unresolved what Muon's five tuned Newton--Schulz steps preserve and whether that suffices to reach a prescribed neural-network training loss. We establish a finite-time training guarantee that accounts for both momentum accumulation before orthogonalization and the tuned finite-step update. For full-batch training of a sufficiently wide two-layer ReLU network with fixed random output weights and a positive-definite limiting neural tangent kernel, we prove that Muon reaches any target empirical squared loss $\varepsilon>0$ with high probability over initialization. For every momentum parameter $\mu\in[0,1)$, a target-dependent constant learning rate proportional to $(1-\mu)\sqrt{\varepsilon}$ yields a hitting-time bound of $O((1-\mu)^{-1}\varepsilon^{-1/2})$, with other problem parameters fixed. The sufficient width is independent of both target accuracy and momentum. The analysis shows that the tuned Newton--Schulz map preserves alignment with the momentum buffer while bounding the update's spectral norm. Control of gradient variation near initialization transfers this alignment to the current gradient, ensuring descent until the target is reached without requiring exact orthogonalization. Numerical experiments support these mechanisms at widths below the sufficient theoretical threshold: gradient-update alignment remains above the analytical reference, and all 30 runs across six widths and five student initializations on a fixed teacher-student dataset reach the target loss while maintaining kernel positivity.
\end{abstract}


\section{Introduction}\label{sec:introduction}

How much can an optimizer alter the gradient and still provably fit the training data? Muon (\emph{MomentUm Orthogonalized by Newton--Schulz}) makes this question concrete: it accumulates gradients in a momentum buffer and then reshapes the buffer's singular values through five tuned Newton--Schulz iterations before updating the weights \citep{jordan2024muon}. This approximately orthogonalized update has proved effective in large-scale language-model training \citep{liu2025muon,team2025kimi}, yet the transformation complicates the connection between the direction applied to the weights and the gradient of the current loss. Momentum carries information from earlier iterates, while finite orthogonalization changes how that information enters the update. Understanding why sufficient descent survives these operations is central to establishing a training-loss guarantee. We study this question for overparameterized two-layer ReLU networks and show that Muon can reach a prescribed loss while retaining both momentum and the tuned five-step Newton--Schulz transformation.

Existing analyses explain several aspects of Muon's update, but leave a gap between its practical construction and a neural-network training-loss guarantee. The exact polar factor admits a geometric interpretation as a steepest-descent direction under a spectral-norm constraint \citep{bernstein2023old,bernstein2024modular}. Analyses incorporating momentum establish stationarity guarantees for smooth objectives, but replace the finite Newton--Schulz map with exact orthogonalization \citep{chen2025muon,shen2025convergence,sato2025muon,chang2025muon,deoliveira2026muon}; \citet{deoliveira2026muon} additionally obtain linear convergence of the objective under a Polyak--{\L}ojasiewicz condition, which is an assumption on the objective rather than a training-loss guarantee for a neural network. \citet{kim2026muonns} analyze finitely many Newton--Schulz iterations, but their guarantees concern classical polynomial iterations rather than the tuned coefficients used by Muon. Separately, overparameterized-network theory establishes conditions under which gradient descent reaches arbitrarily small training loss \citep{du2018gradient,du2019gradient,arora2019fine}. These results do not transfer directly to Muon: a positive-definite tangent kernel controls the gradient through the training residual, but the actual loss decrease depends on the gradient's alignment with the transformed momentum buffer. The missing link is therefore a quantitative descent guarantee that accommodates both accumulated gradients and tuned, approximate orthogonalization. Table~\ref{tab:related} (in the appendix) summarizes the closest results. For gradient descent on this model, \citet{du2018gradient} reach arbitrarily small training loss at a linear rate with width $\Omega(n^6/(\lambda_0^4\delta^3))$, and \citet{arora2019fine} reach loss $\varepsilon$ with width $\Omega(n^7/(\lambda_0^4\nu^2\delta^4\varepsilon^2))$ at initialization scale $\nu=O(\varepsilon\delta/\sqrt n)$. The two closest Muon analyses \citep{chen2025muon,kim2026muonns} guarantee only stationarity. Our result is a training-loss guarantee for Muon, with width $\Omega(n^5/(\lambda_0^4\delta^3))$.

We establish this link by identifying two properties of the finite Newton--Schulz map that suffice for a training-loss guarantee: alignment with its input and a bounded output spectral norm. For Muon's tuned five-step map, we prove both properties uniformly over nonzero momentum buffers, without requiring proximity to the exact polar factor. The remaining challenge is to transfer alignment from the buffer to the current gradient. We show that, within a controlled neighborhood of initialization, gradient variation is sufficiently small for this transfer to hold for every momentum parameter $\mu\in[0,1)$, with a learning rate scaled by $1-\mu$. A uniform neural tangent kernel argument controls the ReLU linearization error and ensures that the iterates remain in this neighborhood until the target loss is reached. Together, these estimates yield an explicit finite-time guarantee for full-batch training of a two-layer ReLU network with fixed random output weights, summarized below.

\begin{theorem*}[Informal statement]
Consider full-batch Muon with momentum $\mu\in[0,1)$ and five tuned Newton--Schulz iterations, training the first layer of a width-$m$ two-layer ReLU network with fixed random output weights. Let the $n$ training inputs have unit norm, and suppose their limiting neural tangent kernel has smallest eigenvalue $\lambda_0>0$. Treat the output-weight scale and the average squared label magnitude as constants. For any target loss $\varepsilon>0$ and failure probability $\delta\in(0,1)$, take
\[
m=\Omega\!\left(\frac{n^5}{\lambda_0^4\delta^3}\right),
\qquad
\eta=\Theta\!\left(\frac{(1-\mu)\sqrt{\lambda_0\varepsilon}}{n}\right).
\]
Then, with probability at least $1-\delta$ over initialization, an iterate with empirical squared loss at most $\varepsilon$ is reached within
\[
T=\mathcal O\!\left(
\frac{n^{3/2}}{(1-\mu)\lambda_0\sqrt{\delta\varepsilon}}
\right)
\]
steps. The precise assumptions and constants are given in Theorem~\ref{thm:main_muon}.
\end{theorem*}

The sufficient width is independent of both the target accuracy
$\varepsilon$ and momentum $\mu$: a smaller target or larger
momentum changes the prescribed learning rate and iteration
budget without requiring a wider network. The result is a
hitting-time guarantee with a target-dependent constant learning
rate; it does not imply acceleration over gradient descent or
a benefit from increasing momentum. By contrast, \citet{wang2021modular} prove that heavy-ball momentum accelerates gradient descent for training a wide ReLU network with one hidden layer; our bound scales with $(1-\mu)^{-1}$ and shows no such gain. Our experiments directly
examine the alignment and stability mechanisms underlying the
proof. Across the tested momentum values, gradient-update
alignment remains above the analytical reference. All 30 runs
across six widths and five student initializations on a fixed
dataset and teacher reach the target loss while maintaining
kernel positivity, and kernel variation decreases with width.
Comparisons on shared momentum buffers further quantify the
alignment and norm properties of tuned Newton--Schulz updates
relative to classical iterations and the exact polar factor.
These findings support the relevance of the proof's mechanisms
at widths many orders of magnitude below the sufficient theoretical
threshold (Section~\ref{sec:experiments}), without establishing the necessity or tightness
of the bounds.
\section{Problem Setup and the Muon Update}\label{sec:model}

For vectors, $\|\cdot\|_2$ denotes the Euclidean norm; for matrices,
$\|\cdot\|_2$, $\|\cdot\|_F$, and $\|\cdot\|_*$ denote the spectral,
Frobenius, and nuclear norms, respectively. We use the Frobenius
inner product $\langle A,B\rangle=\operatorname{tr}(A^\top B)$
and column-wise vectorization.

\subsection{Network and Training Objective}

Let $\{(x_i,y_i)\}_{i=1}^n\subset\mathbb R^d\times\mathbb R$
be fixed training data with $\|x_i\|_2=1$. We consider the
width-$m$ two-layer ReLU network
\begin{equation}
    f(x;W)=\frac{1}{\sqrt m}\sum_{r=1}^m
    a_r\sigma(w_r^\top x),
    \qquad
    W=(w_1,\ldots,w_m)\in\mathbb R^{d\times m},
    \qquad
    \sigma(z)=\max\{0,z\}.
    \label{eq:muon_network}
\end{equation}
Only the first-layer weights $W$ are trained. For a fixed
$\alpha>0$, the output weights are sampled independently as
$a_r\sim\operatorname{Unif}\{\pm\alpha\}$ and held fixed.
The first-layer weights are initialized independently as
$w_r(0)\sim\mathcal N(0,I_d)$, independently of the output
weights. All probability statements concern this random
initialization, conditional on the training data.

Write $y=(y_1,\ldots,y_n)^\top$, let
$u(W)=(f(x_1;W),\ldots,f(x_n;W))^\top$, and define the
residual $e(W)=y-u(W)$. The training objective is the
unnormalized empirical squared loss
\begin{equation}
    L(W)=\frac12\|u(W)-y\|_2^2
        =\frac12\|e(W)\|_2^2.
    \label{eq:muon_loss}
\end{equation}
We adopt the derivative convention
$\sigma'(z)=\mathds 1\{z\ge0\}$, including at $z=0$.
The resulting selected gradient
$G(W)\in\mathbb R^{d\times m}$ has columns
\begin{equation}
    [G(W)]_{\cdot,r}
    =-\frac{a_r}{\sqrt m}\sum_{i=1}^n
    e(W)_i\,\mathds 1\{w_r^\top x_i\ge0\}\,x_i.
    \label{eq:muon_gradient}
\end{equation}
Our objective is to bound the number of full-batch Muon
updates required to reach $L(W)\le\varepsilon$ for a
prescribed $\varepsilon>0$.

\subsection{Momentum and the Newton--Schulz Update}

Muon constructs each update by applying a finite Newton--Schulz
map to a momentum buffer. For a nonzero matrix
$B\in\mathbb R^{d\times m}$, initialize $X_0=B/\|B\|_F$
and perform
\begin{equation}
    X_{k+1}=p_k(X_kX_k^\top)X_k,
    \qquad k=0,\ldots,K-1,
    \label{eq:ns_iteration}
\end{equation}
where $p_k$ is a polynomial. Define
$\mathrm{NS}_K(B)=X_K$ and $\mathrm{NS}_K(0)=0$.
Muon uses $K=5$ with the same tuned polynomial at every
step \citep{jordan2024muon}:
\begin{equation}
    p_k(s)=3.4445-4.7750s+2.0315s^2,
    \qquad k=0,\ldots,4.
    \label{eq:muon_tuned_polynomial}
\end{equation}
This map transforms the singular values of the normalized
buffer while retaining its singular-vector directions.
It does not, in general, return the exact polar factor.
We formulate the analysis for general polynomial schedules
and specialize the resulting bounds to Muon's tuned map.

Algorithm~\ref{alg:muon} specifies the full-batch trajectory.
Momentum is accumulated \emph{before} the Newton--Schulz
transformation; orthogonalizing individual gradients and
then averaging them would define a different update.

\begin{algorithm}[t]
\caption{Full-batch Muon}\label{alg:muon}
\begin{algorithmic}[1]
\Require Initial weights $W_0$, learning rate $\eta>0$,
momentum $\mu\in[0,1)$, target loss $\varepsilon>0$,
iteration budget $T$
\State $B_0\gets0$
\For{$t=1,\ldots,T$}
    \If{$L(W_{t-1})\le\varepsilon$}
        \State \Return $W_{t-1}$
    \EndIf
    \State $G_t\gets G(W_{t-1})$
    \State $B_t\gets\mu B_{t-1}+G_t$
    \State $O_t\gets\mathrm{NS}_5(B_t)$
    \State $W_t\gets W_{t-1}-\eta O_t$
\EndFor
\State \Return $W_T$
\end{algorithmic}
\end{algorithm}

We write $u_t=u(W_t)$ and $e_t=e(W_t)$. Reaching the
target loss is equivalent to $\|e_t\|_2\le\sqrt{2\varepsilon}$.
The analysis concerns the exact-arithmetic Newton--Schulz
map; finite-precision effects and denominator regularization
used in numerical implementations are outside its scope.

\subsection{Kernel and Update Conditions}

The network Jacobian $J(W)\in\mathbb R^{n\times dm}$,
defined using the derivative convention above, has blocks
\begin{equation}
    [J(W)]_{i,(r-1)d+1:rd}
    =\frac{a_r}{\sqrt m}
    \mathds 1\{w_r^\top x_i\ge0\}\,x_i^\top.
    \label{eq:muon_jacobian}
\end{equation}
Consequently,
$\operatorname{vec}(G(W))=-J(W)^\top e(W)$.
Define the empirical neural tangent kernel as
$H(W)=J(W)J(W)^\top$. Under the specified initialization,
its expectation is $\mathbb E[H(W_0)]=\alpha^2H^\infty$,
where
\begin{equation}
    H^\infty_{ij}
    =x_i^\top x_j\,
    \frac{\pi-\arccos(x_i^\top x_j)}{2\pi}.
    \label{eq:muon_hinf}
\end{equation}
We assume that this normalized limiting kernel is
positive definite:
\begin{equation}
    \lambda_0:=\lambda_{\min}(H^\infty)>0.
    \label{eq:muon_assumptions}
\end{equation}
Since $H^\infty_{ii}=1/2$, necessarily
$\lambda_0\le1/2$. Kernel positivity relates the gradient
magnitude to the residual through
$\|G(W)\|_F^2=e(W)^\top H(W)e(W)$.
For Muon, turning this relation into loss decrease also
requires control of the transformed momentum update.

To describe that transformation, define
$\phi_k(x)=x\,p_k(x^2)$. The Newton--Schulz recurrence
acts on the normalized singular values through
\begin{equation}
    \Phi=\phi_{K-1}\circ\cdots\circ\phi_0.
    \label{eq:ns_coefficients}
\end{equation}
In particular, if $B=U\Sigma V^\top$ is a compact
singular value decomposition of a nonzero buffer, then
\[
    \mathrm{NS}_K(B)
    =U\,\Phi\!\left(\frac{\Sigma}{\|B\|_F}\right)V^\top,
\]
where $\Phi$ acts entrywise on the diagonal.

\begin{assumption}[Newton--Schulz schedule]\label{ass:ns}
There exist constants $\kappa>0$ and $M>0$ such that
\begin{equation}
    \kappa x\le\Phi(x)\le M
    \qquad\text{for all }x\in[0,1].
    \label{eq:ns_constants}
\end{equation}
\end{assumption}

Taking $x=1$ in \eqref{eq:ns_constants} gives $\kappa\le\Phi(1)\le M$;
we use $\kappa\le M$ throughout.

The constant $\kappa$ controls alignment with the input
buffer, while $M$ bounds the update's spectral norm.
Indeed, for every nonzero $B$,
\[
    \langle B,\mathrm{NS}_K(B)\rangle
    \ge\kappa\|B\|_F,
    \qquad
    \|\mathrm{NS}_K(B)\|_2\le M.
\]
These conditions do not require the transformed singular
values to be close to one. Proposition~\ref{prop:jordan_constants}
establishes their validity for Muon's tuned five-step map.
The next section uses these bounds together with kernel
stability to transfer alignment from the momentum buffer
to the current gradient and establish the training guarantee.
\section{Finite-Time Training Guarantee}
\label{sec:main_result}

We establish a training-loss guarantee for any
Newton--Schulz schedule satisfying
Assumption~\ref{ass:ns}, including Muon's tuned
five-step map. Starting from $B_0=0$, consider
\begin{equation}
    G_t=G(W_{t-1}),\quad
    B_t=\mu B_{t-1}+G_t,\quad
    O_t=\mathrm{NS}_K(B_t),\quad
    W_t=W_{t-1}-\eta O_t.
    \label{eq:muon_trajectory}
\end{equation}
For $\delta\in(0,1)$, define the initial-residual bound
\begin{equation}
    R_0:=
    \sqrt{\frac{3(\|y\|_2^2+n\alpha^2/2)}{\delta}}.
    \label{eq:muon_constants}
\end{equation}

\begin{theorem}[Finite-time training guarantee]
\label{thm:main_muon}
Consider the model and initialization of
Section~\ref{sec:model}, and suppose
Assumption~\ref{ass:ns} holds with constants
$\kappa$ and $M$. Fix $\varepsilon>0$,
$\delta\in(0,1)$, and $\mu\in[0,1)$. If
\begin{equation}
    m\ge
    5\cdot10^{12}
    \left(\frac{M}{\kappa}\right)^8
    \frac{n^4}{\lambda_0^4\delta^3}
    \left(n+\frac{\|y\|_2^2}{\alpha^2}\right),
    \label{eq:muon_width}
\end{equation}
run \eqref{eq:muon_trajectory} with
\begin{align}
    \eta
    &=(1-\mu)
    \frac{\kappa\sqrt{\lambda_0\varepsilon}}
         {8M^2n\alpha},
    \label{eq:muon_momentum_condition}\\
    T
    &=\left\lceil
    \frac{128M^2nR_0}
         {(1-\mu)\kappa^2\lambda_0\sqrt{\varepsilon}}
    \right\rceil.
    \label{eq:muon_budget}
\end{align}
Then, with probability at least $1-\delta$ over
initialization,
\begin{equation}
    \min_{0\le t\le T}L(W_t)\le\varepsilon.
    \label{eq:muon_hitting_conclusion}
\end{equation}
Thus, stopping at the first iterate with loss at
most $\varepsilon$ requires no more than $T$ updates.
\end{theorem}

For Muon's tuned coefficients,
Proposition~\ref{prop:jordan_constants} permits
$\kappa=2/3$ and $M=5/4$. The sufficient width is
independent of both $\varepsilon$ and $\mu$,
whereas the learning rate and iteration budget
depend explicitly on them. The result is a
target-dependent hitting-time guarantee; it does
not assert convergence to zero under one fixed
learning rate. The proof is given in
Section~\ref{sec:main_proof}.

\medskip\noindent\textbf{Scope: what is specific to Muon.}
The theorem uses the Newton--Schulz map only through $\kappa$ and $M$.
It therefore also covers $K=0$, where $\mathrm{NS}_0(B)=B/\|B\|_F$ and
$\Phi(x)=x$: Frobenius-normalized momentum gradient descent satisfies
Assumption~\ref{ass:ns} with $\kappa=M=1$. What is specific to Muon is that
its tuned five-step map satisfies the assumption, with $\kappa=2/3$ and
$M=5/4$, without its transformed singular values being close to one, that is,
without approximating the polar factor.

\medskip\noindent\textbf{Why $\varepsilon^{-1/2}$.}
Since $\langle B_t,O_t\rangle\le\|B_t\|_F\|O_t\|_F$, the alignment bound gives
$\|O_t\|_F\ge\kappa$ whenever $B_t\ne0$: the update does not shrink with the
gradient. With a fixed $\eta$, the descent inequality
\eqref{eq:muon_squared_descent} then certifies progress only down to a residual
of order $\eta$, because of its term $nM^2\alpha^2\eta^2$. Reaching loss
$\varepsilon$ therefore requires $\eta\propto\sqrt\varepsilon$, and the
iteration count scales as $1/\eta$. Gradient descent differs because its steps
shrink with the gradient, which gives a linear rate with a constant step size
\citep{du2018gradient}.

\subsection{Stability Near Initialization}

The analysis controls the trajectory within
\begin{equation}
    \mathbb B:=
    \left\{
    W:\|W-W_0\|_F
    \le
    \frac{18M\sqrt n\,R_0}
         {\kappa\alpha\sqrt{\lambda_0}}
    \right\}.
    \label{eq:muon_radius_margin}
\end{equation}
Under \eqref{eq:muon_width},
Lemma~\ref{lem:muon_initialization} establishes an
event $\mathcal E$ of probability at least
$1-\delta$ controlling the initial residual,
kernel concentration, and the number of
neuron--input pairs near an activation boundary.
On this event, Lemma~\ref{lem:muon_ball} gives,
uniformly for $W,V\in\mathbb B$,
\[
    \lambda_{\min}(H(W))
    \ge\frac{\alpha^2\lambda_0}{2},
    \qquad
    \|u(V)-u(W)-J(W)\operatorname{vec}(V-W)\|_2
    \le
    \frac{\kappa\alpha\sqrt{\lambda_0}}{18M}
    \|V-W\|_2.
\]
The full statements and proofs appear in
Appendices~\ref{app:init} and~\ref{app:ball}.
Kernel positivity relates the gradient magnitude
to the residual, while the spectral-norm remainder
bound controls the prediction error produced by
an update. The trajectory containment argument
uses the update norm and total iteration budget,
without first assuming loss decrease.

\subsection{From Momentum Alignment to Descent}

By \eqref{eq:muon_gradient}, every gradient and
momentum buffer has its column space in
$\operatorname{span}\{x_1,\ldots,x_n\}$.
The Newton--Schulz conditions therefore yield
\begin{equation}
    \langle B_t,O_t\rangle\ge\kappa\|B_t\|_F,
    \qquad
    \|O_t\|_2\le M,
    \qquad
    \|O_t\|_F\le M\sqrt n.
    \label{eq:ns5_matrix_bounds}
\end{equation}
The Frobenius bound follows from the rank bound
$\operatorname{rank}(B_t)\le n$; the complete
statement is Lemma~\ref{lem:ns5} in
Appendix~\ref{app:ns}. These bounds control
alignment with the buffer. The central step
transfers that alignment to the current gradient.

\begin{lemma}[Gradient-update alignment]
\label{lem:muon_momentum}
Assume $\mathcal E$, let
$W_0,\ldots,W_{t-1}\in\mathbb B$, and suppose
$\|e_{t-1}\|_2>\sqrt{2\varepsilon}$. If
\[
    \eta\le
    (1-\mu)
    \frac{\kappa\sqrt{\lambda_0\varepsilon}}
         {8M^2n\alpha},
\]
then
\begin{equation}
    \langle G_t,O_t\rangle
    \ge\frac{\kappa}{2}\|G_t\|_F.
    \label{eq:muon_stale_condition}
\end{equation}
\end{lemma}

The proof controls the discrepancy between the
normalized momentum buffer and the current
gradient using the stability estimates within
$\mathbb B$; see Appendix~\ref{app:momentum}.
The learning-rate factor $1-\mu$ compensates
for the accumulation of older gradients.

\begin{lemma}[One-step residual decrease]
\label{lem:muon_step}
Assume $\mathcal E$, let $W_{t-1},W_t\in\mathbb B$,
and suppose \eqref{eq:muon_stale_condition} holds.
Then
\begin{equation}
    \|e_t\|_2^2
    \le
    \|e_{t-1}\|_2^2
    -\frac{\kappa\alpha\eta\sqrt{\lambda_0}}{4}
     \|e_{t-1}\|_2
    +nM^2\alpha^2\eta^2.
    \label{eq:muon_squared_descent}
\end{equation}
If, additionally,
\begin{equation}
    \eta\le
    \frac{\kappa\sqrt{\lambda_0}\|e_{t-1}\|_2}
         {8M^2n\alpha},
    \label{eq:muon_local_stepsize}
\end{equation}
then
\begin{equation}
    \|e_t\|_2
    \le\|e_{t-1}\|_2
    -\frac{\kappa\alpha\eta\sqrt{\lambda_0}}{16}.
    \label{eq:muon_linear_descent}
\end{equation}
\end{lemma}

The proof appears in Appendix~\ref{app:step}.
The prescribed learning rate satisfies
\eqref{eq:muon_local_stepsize} whenever
$L(W_{t-1})>\varepsilon$. Thus, containment in
$\mathbb B$ and gradient-update alignment yield
a fixed decrease in residual norm until the
target is reached; summing this decrease gives
the budget in \eqref{eq:muon_budget}.

The exact polar factor maximizes buffer alignment
per unit spectral norm
(Proposition~\ref{prop:polar_optimal} in the appendix).
Our guarantee requires only the bounds in
\eqref{eq:ns5_matrix_bounds}, which the tuned
finite-step map satisfies without needing to
approximate the polar factor closely.
\section{Proof of the Finite-Time Guarantee}
\label{sec:main_proof}

The proof first establishes trajectory containment using
only the update norm and iteration budget. This makes
the uniform kernel bounds available without assuming
loss decrease. Gradient-update alignment then yields
a fixed residual decrease until the target is reached.

\begin{proof}[Proof of Theorem~\ref{thm:main_muon}]
By Lemma~\ref{lem:muon_initialization}, the event
$\mathcal E$ has probability at least $1-\delta$.
Fix an initialization in $\mathcal E$, so that
$\|e_0\|_2\le R_0$. If $R_0\le\sqrt{2\varepsilon}$,
the initial iterate already satisfies the conclusion.
Hence assume $R_0>\sqrt{2\varepsilon}$.

\paragraph{Trajectory containment.}
Every momentum buffer has its column space in
$\operatorname{span}\{x_1,\ldots,x_n\}$. Thus,
Lemma~\ref{lem:ns5} gives
\begin{equation}
    \|W_t-W_{t-1}\|_F
    =\eta\|O_t\|_F
    \le\eta M\sqrt n.
    \label{eq:proof_step_movement}
\end{equation}
Using $\lceil z\rceil\le z+1$ and the prescribed
learning rate and budget,
\begin{equation}
    T\eta
    \le
    \frac{16R_0}{\kappa\alpha\sqrt{\lambda_0}}+\eta
    \le
    \frac{17R_0}{\kappa\alpha\sqrt{\lambda_0}}.
    \label{eq:proof_total_budget}
\end{equation}
For the last inequality, observe that
\[
    \eta\kappa\alpha\sqrt{\lambda_0}
    =
    \frac{(1-\mu)\kappa^2\lambda_0\sqrt{\varepsilon}}
         {8M^2n}
    \le\frac{\sqrt{\varepsilon}}{16}
    <R_0,
\]
where $\kappa\le M$, $\lambda_0\le1/2$, and $n\ge1$.
Consequently, for every $0\le t\le T$,
\begin{equation}
    \|W_t-W_0\|_F
    \le t\eta M\sqrt n
    \le
    \frac{17M\sqrt n\,R_0}
         {\kappa\alpha\sqrt{\lambda_0}}
    <
    \frac{18M\sqrt n\,R_0}
         {\kappa\alpha\sqrt{\lambda_0}}.
    \label{eq:proof_ball_membership}
\end{equation}
Thus $W_0,\ldots,W_T\in\mathbb B$, and
Lemma~\ref{lem:muon_ball} applies throughout the budget.

\paragraph{Descent until the target is reached.}
Suppose, for contradiction, that
\begin{equation}
    \|e_t\|_2>\sqrt{2\varepsilon}
    \qquad\text{for every }0\le t\le T.
    \label{eq:proof_no_hit}
\end{equation}
For each $1\le t\le T$, trajectory containment and
the prescribed learning rate allow us to apply
Lemma~\ref{lem:muon_momentum}, yielding
$\langle G_t,O_t\rangle\ge(\kappa/2)\|G_t\|_F$.
Moreover,
\[
    \eta
    \le
    \frac{\kappa\sqrt{\lambda_0}\sqrt{\varepsilon}}
         {8M^2n\alpha}
    <
    \frac{\kappa\sqrt{\lambda_0}\|e_{t-1}\|_2}
         {8M^2n\alpha}.
\]
Hence Lemma~\ref{lem:muon_step} gives
\begin{equation}
    \|e_t\|_2
    \le\|e_{t-1}\|_2
    -\frac{\kappa\alpha\eta\sqrt{\lambda_0}}{16}.
    \label{eq:proof_induction_claim}
\end{equation}
Summing over the budget and substituting
\eqref{eq:muon_momentum_condition} yields
\begin{align}
    \|e_T\|_2
    &\le
    R_0-
    T\frac{\kappa\alpha\eta\sqrt{\lambda_0}}{16}
    \nonumber\\
    &=
    R_0-
    T\frac{(1-\mu)\kappa^2\lambda_0\sqrt{\varepsilon}}
           {128M^2n}
    \le0,
    \label{eq:proof_telescoping}
\end{align}
where the last inequality follows from
\eqref{eq:muon_budget}. This contradicts
\eqref{eq:proof_no_hit}. Therefore some $t\le T$
satisfies $L(W_t)\le\varepsilon$.
Since the argument holds on $\mathcal E$, the
conclusion has probability at least $1-\delta$.
\end{proof}
\section{Numerical Evaluation of the Theory}
\label{sec:experiments}

We examine the accuracy and momentum dependence of
the training dynamics, gradient-update alignment,
and kernel stability across widths. The experiments
use widths far below the sufficient condition in
Theorem~\ref{thm:main_muon}: for our dataset,
\eqref{eq:muon_width} requires $m\gtrsim2\times10^{27}$ and
\eqref{eq:muon_budget} gives $T\approx1.6\times10^{8}$ at
$\mu=0.95$ and $\varepsilon=10^{-2}$, whereas the width sweep uses widths $64$--$2000$
and its hitting times at the same $\mu$ and $\varepsilon$ are below $7.3\times10^{5}$. The theoretical
budget and stability bounds thus serve as references
rather than certified guarantees for these runs.
Detailed protocols are provided in Appendix~\ref{app:experimental_details}.

\paragraph{Setup.}
We train a two-layer ReLU student on unit-normalized
Gaussian inputs with labels generated by an independent,
fixed ReLU teacher. Only the student's first layer
is trained, using full-batch Muon with five tuned
Newton--Schulz steps evaluated in double precision.
Unless otherwise stated, $n=20$, $d=100$, $m=1000$,
$\alpha=1$, and $\delta=0.05$. We compute
$\lambda_0$ from the limiting kernel and $R_0$ from
\eqref{eq:muon_constants}, then prescribe $\eta$ and
the reference budget $T$ using
\eqref{eq:muon_momentum_condition} and
\eqref{eq:muon_budget}. The accuracy panels use
$\lambda_0\approx0.34$ and $R_0\approx26$; the saved runs used
in the remaining diagnostics have $\lambda_0=0.335$ and $R_0=30.2$.
We record the first target crossing,
$t_{\mathrm{hit}}=\min\{t:L(W_t)\le\varepsilon\}$.

\paragraph{Accuracy and momentum scaling.}

\begin{figure}[t]
    \centering
    \includegraphics[width=0.68\linewidth]{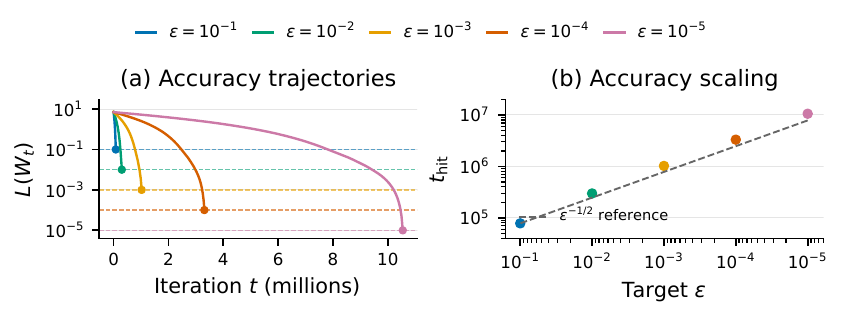}\\[0.2em]
    \includegraphics[width=0.68\linewidth]{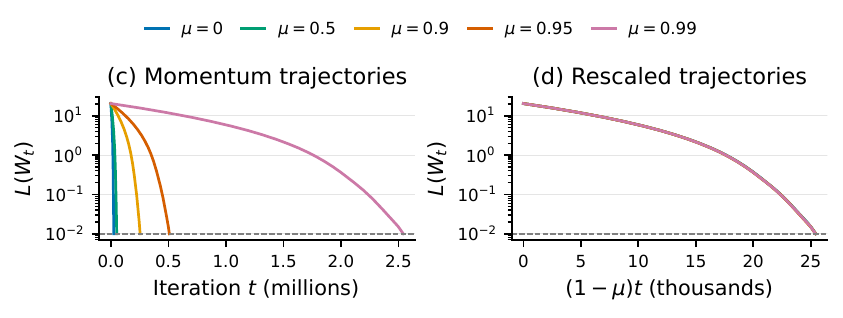}
    \caption{
    \textbf{Training loss across accuracy targets and momentum values.}
    (a) Loss curves for five target values of $\varepsilon$ at $\mu=0.95$;
    dots mark the first iteration reaching each target.
    (b) Hitting time versus $\varepsilon$, with an $\varepsilon^{-1/2}$
    reference line. (c) Loss curves for five momentum values at target
    $10^{-2}$, using $\eta=(1-\mu)\eta_0$ with
    $\eta_0=1.54\times10^{-4}$. (d) The curves in (c) plotted against
    $(1-\mu)t$.
    Panels (a,b) use a separate data draw ($\lambda_0\approx0.34$,
    $R_0\approx26$); panels (c,d) use the saved draw ($\lambda_0=0.335$,
    $R_0=30.2$).
    }
    \label{fig:scaling}
\end{figure}

The accuracy sweep fixes $\mu=0.95$ and uses
$\eta\propto\sqrt{\varepsilon}$. The panels show target attainment
before the reference budget and a trend consistent with its
$\varepsilon^{-1/2}$ dependence (Figure~\ref{fig:scaling}a,b).
This compares separate runs with different learning rates.

With $\eta=(1-\mu)\eta_0$, larger momentum requires more iterations,
and plotting against $(1-\mu)t$ approximately collapses the trajectories
(Figure~\ref{fig:scaling}c,d). This reflects the joint change in momentum
and learning rate. A separate fixed-step comparison uses $\eta=\eta_0$
for $50{,}000$ iterations: all three seeds at each displayed momentum
reach $10^{-2}$ in $14{,}680$--$25{,}444$ updates. Thus the factor
$1-\mu$ remains a sufficient choice; these runs do not establish necessity.

\paragraph{Gradient-update alignment.}

We directly examine the mechanism of
Lemma~\ref{lem:muon_momentum} by measuring
gradient-update alignment and the discrepancy
between the current gradient and normalized
momentum:
\begin{equation}
    A_t=\frac{\langle G_t,O_t\rangle}{\|G_t\|_F},
    \qquad
    D_t=\frac{\|B_t/s_t-G_t\|_F}{\|G_t\|_F},
    \qquad
    s_t=\frac{1-\mu^t}{1-\mu}.
    \label{eq:alignment_diagnostics}
\end{equation}
Here $B_t/s_t$ is the weighted average of stored
gradients. The theorem gives the reference
$A_t\ge\kappa/2=1/3$ under its assumptions, while
$D_t$ measures the buffer discrepancy involved
in transferring alignment to the current gradient.
We record both quantities before target attainment
for several momentum values.

\begin{figure}[t]
    \centering
    \includegraphics[width=0.7\linewidth]{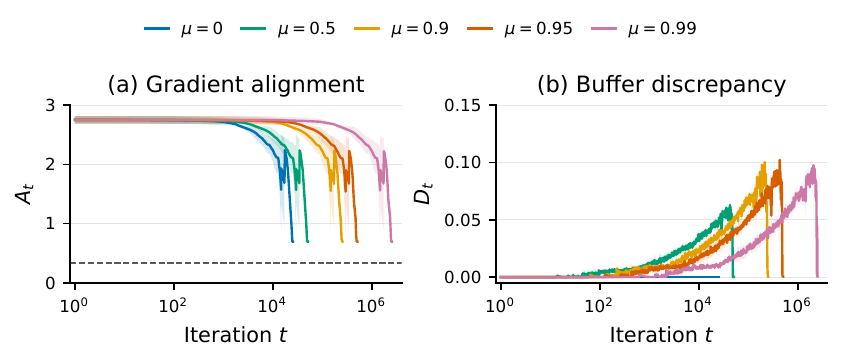}
    \caption{
    \textbf{Gradient alignment and momentum discrepancy.}
    Both panels plot iteration $t$ (logarithmic axis) for five momentum values
    $\mu$, each run until it first reaches the target $10^{-2}$.
    (a) Alignment $A_t=\langle G_t,O_t\rangle/\|G_t\|_F$ between the gradient
    and the update; the dashed line is the reference $\kappa/2=1/3$.
    (b) Buffer discrepancy $D_t=\|B_t/s_t-G_t\|_F/\|G_t\|_F$, which is
    identically zero at $\mu=0$.}
    \label{fig:alignment}
\end{figure}

Alignment remains about $2.7$ for most of training and declines near
the target; its smallest recorded value is $0.695$ (at $\mu=0.99$), above $1/3$.
The discrepancy is zero at $\mu=0$ and has run-wise maxima in
$[0.066,0.073]$ at $\mu=0.5$ and $[0.105,0.124]$ for $\mu\ge0.9$.
The alignment decrease also occurs at $\mu=0$, so it cannot be attributed
solely to momentum discrepancy.

\paragraph{Width and kernel stability.}

We vary student width while keeping the dataset,
teacher, target, momentum, and prescribed learning
rate fixed. Across repeated student initializations,
we record target attainment and monitor
\begin{equation}
    K_t=
    \frac{\lambda_{\min}(H(W_t))}{\alpha^2\lambda_0},
    \qquad
    S_t=
    \frac{\|H(W_t)-H(W_0)\|_2}{\|H(W_0)\|_2}.
    \label{eq:width_diagnostics}
\end{equation}
The reference $K_t=1/2$ corresponds to the kernel
lower bound used in the proof; $S_t$ measures
relative kernel variation. Figure~\ref{fig:width_stability}
summarizes these quantities over each run.
We also measure displacement relative to the analytical radius
in \eqref{eq:muon_radius_margin} (Figure~\ref{fig:appendix_displacement}).

\begin{figure}[t]
    \centering
    \includegraphics[width=0.7\linewidth]{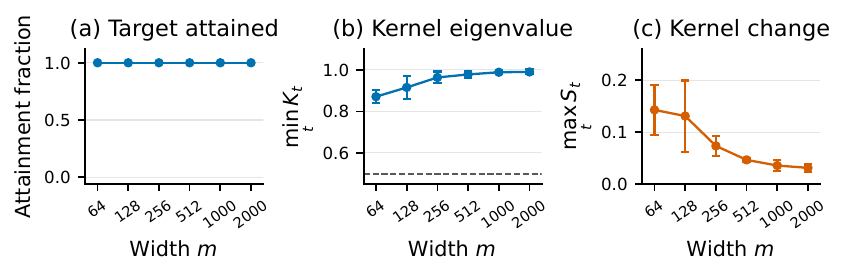}
    \caption{
    \textbf{Training success and kernel stability
    across widths.}
    (a) Target attainment within $1{,}414{,}000$ updates;
    (b) minimum $K_t$, with reference $1/2$; (c) maximum $S_t$.
    Points and bars show mean $\pm1$ sample standard deviation over
    five seeds per width, measured through target attainment.
    }
    \label{fig:width_stability}
\end{figure}

Every one of these $30$ runs reaches
$10^{-2}$. The run-wise minimum of $K_t$ has
mean $0.871\pm 0.030$ at width $64$ and
$0.990\pm 0.012$ at width $2000$, and the
smallest value on any run is $0.826$, above the
reference $1/2$. The run-wise maximum of $S_t$
falls from $0.143\pm 0.047$ at width $64$ to
$0.031\pm 0.007$ at width $2000$. No run fails
to reach the target.
These measurements examine behavior below the
sufficient width threshold; they do not establish
that threshold's necessity or tightness.

\paragraph{Finite maps and parameter movement.}

Figure~\ref{fig:ns_comparison} compares the tuned five-step map, five
classical cubic Newton--Schulz steps, and the exact polar factor on $59$ shared
training buffers. For a buffer $B$ with update $O$, we record the alignment
$a=\langle B,O\rangle/\|B\|_F$, the spectral norm $s=\|O\|_2$, and the ratio
$q=\langle B,O\rangle/(\|B\|_*\|O\|_2)$, whose maximum $1$ is attained by the
polar factor. The tuned map satisfies $a\ge2/3$ and $s\le5/4$ on every buffer,
as Assumption~\ref{ass:ns} requires. The median $q$ is $0.764$ for the tuned map
and $0.860$ for the classical map. These are diagnostics of the maps; they do
not compare training speed or distance to the polar factor.

\begin{figure}[t]
\centering
\includegraphics[width=0.74\linewidth]{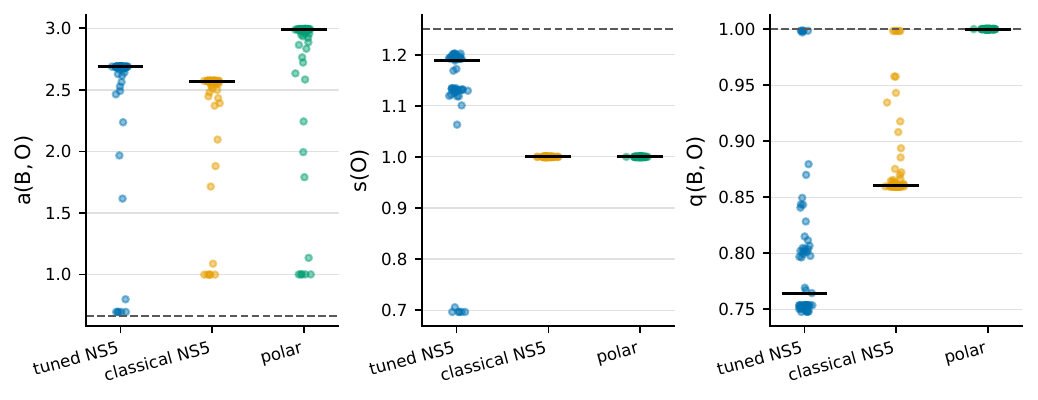}
\caption{\textbf{Finite maps on shared training buffers.}
Left to right: buffer alignment $a$, spectral norm $s$, and normalized
alignment per spectral norm $q$. Points are $59$ saved buffers from
student seed $10$; black ticks mark medians. Dashed references are
$2/3$, $5/4$, and the polar optimum $1$, respectively.
Classical NS5 uses the cubic recurrence.}
\label{fig:ns_comparison}
\end{figure}

\begin{figure}[t]
\centering
\begin{minipage}[c]{0.38\linewidth}
\includegraphics[width=\linewidth]{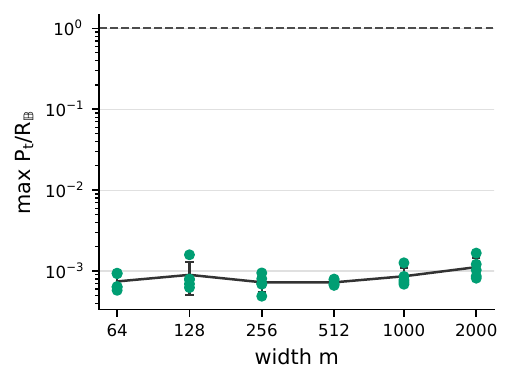}
\end{minipage}\hfill
\begin{minipage}[c]{0.59\linewidth}
\caption{\textbf{Movement within the analytical radius.}
Maximum $\|W_t-W_0\|_F/R_{\mathbb B}$ up to target attainment, with
$R_{\mathbb B}$ from \eqref{eq:muon_radius_margin}.
Points show five seeds per width; the line and bars show their mean
and $\pm1$ sample standard deviation. The dashed line is the ball boundary.}
\label{fig:appendix_displacement}
\end{minipage}
\end{figure}

In every width-sweep run, the target is reached with displacement at most
$1.67\times10^{-3}R_{\mathbb B}$, where $R_{\mathbb B}=7.88\times10^3$. The
iterates thus remain well inside $\mathbb B$ and the kernel remains positive,
consistent with the containment step of the proof. These measurements do not
verify the uniform Jacobian and remainder bounds at these widths.

\section{Conclusion}
\label{sec:conclusion}

We proved a finite-time training-loss guarantee for full-batch Muon with
momentum followed by five tuned Newton--Schulz iterations. For sufficiently wide
two-layer ReLU networks with fixed random output weights, a learning rate
proportional to $(1-\mu)\sqrt{\varepsilon}$ attains any prescribed loss
$\varepsilon$ within $\mathcal O\bigl((1-\mu)^{-1}\varepsilon^{-1/2}\bigr)$
iterations, with probability at least $1-\delta$ over initialization. The proof
uses only two properties of the finite-step map, alignment with the momentum
buffer and a bounded spectral norm. The tuned map has both without approximating
the polar factor, so exact orthogonalization is unnecessary.

The experiments are consistent with these mechanisms at widths far below the
sufficient threshold ($m\gtrsim2\times10^{27}$ for our dataset):
gradient-update alignment stays above $1/3$, the kernel ratio $K_t$ above $1/2$,
displacement below $1.7\times10^{-3}R_{\mathbb B}$, and all $30$ width-sweep
runs reach the target. They do not test necessity or tightness. On shared
buffers, classical Newton--Schulz has a higher median alignment per spectral
norm than the tuned map, which does not by itself imply faster training.

\textbf{Future work.} Four directions remain: (i) sharper width and momentum
requirements, and whether momentum can accelerate Muon as it does gradient
descent in \citet{wang2021modular}; (ii) analysis beyond the near-initialization regime, in
which the features move; (iii) minibatch gradients and finite-precision
arithmetic; and (iv) deeper networks and joint training of both layers.

\clearpage
\section*{Reproducibility Statement}
The network, initialization, selected ReLU derivative, and full-batch loss are
specified in Section~\ref{sec:model}. Algorithm~\ref{alg:muon} is the Muon
update analyzed in the paper, and
equations~\eqref{eq:ns_iteration}--\eqref{eq:muon_tuned_polynomial} fix all five
Newton--Schulz steps; Assumption~\ref{ass:ns} states the only properties of
the Newton--Schulz steps that the proof uses, and
Appendix~\ref{app:jordan_coefficients} verifies them for the coefficients of
\citet{jordan2024muon}. Equations
\eqref{eq:muon_constants}--\eqref{eq:muon_budget} specify the width, momentum
range, learning rate, and update budget. No empirical result is used in the
proof.

\clearpage
\appendix
\section{Comparison with Prior Guarantees}
\label{app:related_guarantees}

\begin{table}[H]
\caption{Non-asymptotic guarantees for wide-network training and Muon.
The first eight rows follow the raw gradient: GD for the first five and SGD
for the next three. Networks have depth $2$ except in
\citet{du2019gradient,zou2019improved,allen2019convergence,cao2019generalization},
which allow depth $>2$. The last three rows use momentum followed by,
respectively, the exact polar factor, classical Newton--Schulz (NS), and five
tuned NS steps. The two prior Muon results concern general $L$-smooth
objectives. Population-risk and stationarity guarantees are not directly
comparable to training-loss guarantees. A dash denotes an unstated explicit
bound. For \citet{arora2019fine}, $\nu=O(\varepsilon\delta/\sqrt n)$ is the
initialization scale. For our row, $\alpha$ and $\|y\|_2^2/n$ are treated as constants and
$T=\mathcal O\bigl(n^{3/2}/((1-\mu)\lambda_0
\sqrt{\delta\varepsilon})\bigr)$.}
\label{tab:related}
\centering
\renewcommand{\arraystretch}{1.12}
\begin{tabular*}{\linewidth}{@{\extracolsep{\fill}}lll@{}}
\toprule
Reference & Guarantee & Width\\
\midrule
\citet{du2018gradient} & $L(W_T)\to0$, linear & $\Omega(n^6/(\lambda_0^4\delta^3))$\\
\citet{arora2019fine} & $L(W_T)\le\varepsilon$ & $\Omega(n^7/(\lambda_0^4\nu^2\delta^4\varepsilon^2))$\\
\citet{ji2019polylogarithmic} & test error $\le\varepsilon$ & $\operatorname{polylog}(1/\varepsilon)$\\
\citet{du2019gradient} & $L(W_T)\to0$, linear & -- ($\operatorname{poly}(n)$)\\
\citet{zou2019improved} & $L(W_T)\to0$, faster & $<$ Allen-Zhu et al.\\
\citet{li2018learning} & pop.\ risk $\to0$ & margin-dependent\\
\citet{allen2019convergence} & $L(W_T)\to0$, poly time & $\Omega(n^{24})$\\
\citet{cao2019generalization} & pop.\ risk $\le\widetilde{\mathcal O}(n^{-1/2})$ & sufficiently wide\\
\citet{chen2025muon} & KKT score $\mathcal O(1/\sqrt T)$ & --\\
\citet{kim2026muonns} & stationarity & --\\
\citet{sato2025muon} & stationarity & --\\
\citet{chang2025muon} & stationarity & --\\
\citet{deoliveira2026muon} & stationarity; linear under PL & --\\
\midrule
\textbf{This work} & $\min_{0\le t\le T}L(W_t)\le\varepsilon$ &
$\Omega\bigl(n^5/(\lambda_0^4\delta^3)\bigr)$\\
\bottomrule
\end{tabular*}
\end{table}

\section{Supporting Lemmas and Proofs}\label{sec:lemma_proofs}

\subsection{Initialization}\label{app:init}

\begin{lemma}[Initialization event]\label{lem:muon_initialization}
Let
\begin{align}
 \mathcal E_1
 &:=\{\|e_0\|_2\le R_0\},\nonumber\\
 \mathcal E_2
 &:=\left\{\|H(W_0)-\alpha^2H^\infty\|_2
 \le\frac{\alpha^2\lambda_0}{4}\right\},\nonumber\\
 \mathcal E_3
 &:=\left\{\frac1m\sum_{i=1}^n\sum_{r=1}^m
 \mathds1\Bigl\{|w_r(0)^\top x_i|\le\frac{\delta\kappa^2\lambda_0}{1620M^2n}\Bigr\}
 \le\frac{\kappa^2\lambda_0}{648M^2}\right\}.
 \label{eq:muon_good_events}
\end{align}
Under the width condition \eqref{eq:muon_width}, each complement has
probability at most $\delta/3$. Hence
$\mathcal E:=\mathcal E_1\cap\mathcal E_2\cap\mathcal E_3$ satisfies
$\Pr(\mathcal E)\ge1-\delta$.
\end{lemma}

\begin{proof}[Proof of Lemma~\ref{lem:muon_initialization}]
We show that each of the three events fails with probability at most
$\delta/3$; a union bound then gives $\Pr(\mathcal E)\ge1-\delta$.

\noindent\textbf{Part 1.}
For $Z\sim\mathcal N(0,1)$, $\mathbb E\sigma(Z)^2=\mathbb E[Z^2\mathds1\{Z\ge0\}]=1/2$.
Since the $a_r$ are independent of $W_0$ with $\mathbb E a_r=0$ and
$a_r^2=\alpha^2$, and $w_r(0)^\top x_i\sim\mathcal N(0,1)$,
\begin{equation}
 \mathbb E u(W_0)_i=0,\qquad
 \mathbb E u(W_0)_i^2=\frac1m\sum_{r=1}^m\alpha^2\,\mathbb E\sigma(w_r(0)^\top x_i)^2
 =\frac{\alpha^2}{2}.
 \label{eq:appendix_initial_moments}
\end{equation}
Then
\begin{align}
 \Pr(\mathcal E_1^c)
 &=\Pr(\|e_0\|_2^2>R_0^2)\nonumber\\
 &\le\frac{\mathbb E\|e_0\|_2^2}{R_0^2}
   \why{Markov's inequality}\nonumber\\
 &=\frac{\|y\|_2^2+n\alpha^2/2}{R_0^2}
   \why{cross term vanishes by \eqref{eq:appendix_initial_moments}}\nonumber\\
 &=\frac\delta3
   \why{definition \eqref{eq:muon_constants} of $R_0$}.
 \label{eq:appendix_E1_probability}
\end{align}

\noindent\textbf{Part 2.}
For fixed $i,j$,
\begin{equation}
 \frac{H(W_0)_{ij}}{\alpha^2}
 =\frac1m\sum_{r=1}^m Z_r^{ij},\qquad
 Z_r^{ij}:=x_i^\top x_j\,
 \mathds1\{w_r(0)^\top x_i\ge0,\ w_r(0)^\top x_j\ge0\}.
 \label{eq:appendix_kernel_summands}
\end{equation}
The $Z_r^{ij}$ are independent over $r$, have mean $H^\infty_{ij}$ by
\eqref{eq:muon_hinf}, and take values in an interval of length
$|x_i^\top x_j|\le1$. Since $\log z\le z$, $\kappa\le M$, $\lambda_0\le\frac12$
and $\delta<1$, the width condition \eqref{eq:muon_width} gives
$m\ge48n^4/(\lambda_0^2\delta)\ge(8n^2/\lambda_0^2)\log(6n^2/\delta)$. Hence
\begin{align}
 \Pr\Bigl(\max_{i,j}\Bigl|\frac{H(W_0)_{ij}}{\alpha^2}-H^\infty_{ij}\Bigr|>\frac{\lambda_0}{4n}\Bigr)
 &\le n^2\cdot2\exp\Bigl(-\frac{m\lambda_0^2}{8n^2}\Bigr)
   \why{union bound, Hoeffding's inequality}\nonumber\\
 &\le\frac\delta3
   \why{$m\ge(8n^2/\lambda_0^2)\log(6n^2/\delta)$}.
 \label{eq:appendix_hoeffding}
\end{align}
Outside this event, $D:=H(W_0)/\alpha^2-H^\infty$ satisfies
\begin{align}
 \|D\|_2
 &\le\|D\|_F
   \why{operator norm $\le$ Frobenius norm}\nonumber\\
 &\le n\max_{i,j}|D_{ij}|
   \why{$D$ has $n^2$ entries}\nonumber\\
 &\le\frac{\lambda_0}{4},
 \label{eq:appendix_kernel_operator}
\end{align}
which is $\mathcal E_2$ after multiplying by $\alpha^2$. Hence
$\Pr(\mathcal E_2^c)\le\delta/3$.

\noindent\textbf{Part 3.}
The density of $\mathcal N(0,1)$ is at most $1/\sqrt{2\pi}$, so
$\Pr\bigl(|w_r(0)^\top x_i|\le\frac{\delta\kappa^2\lambda_0}{1620M^2n}\bigr)
\le\sqrt{2/\pi}\,\frac{\delta\kappa^2\lambda_0}{1620M^2n}$. Let $N$ be the double
sum in $\mathcal E_3$. Then
\begin{align}
 \Pr(\mathcal E_3^c)
 &=\Pr\Bigl(\frac Nm>\frac{\kappa^2\lambda_0}{648M^2}\Bigr)\nonumber\\
 &\le\frac{648M^2}{\kappa^2\lambda_0}\,\mathbb E\frac Nm
   \why{Markov's inequality}\nonumber\\
 &\le\frac{648M^2}{\kappa^2\lambda_0}\cdot n\sqrt{\frac2\pi}\,
   \frac{\delta\kappa^2\lambda_0}{1620M^2n}
   \why{linearity of expectation}\nonumber\\
 &\le\frac\delta3
   \why{$\sqrt{2/\pi}\le\tfrac56$}.
 \label{eq:appendix_E3_probability}
\end{align}
The union bound over Parts~1--3 completes the proof.
\end{proof}

\subsection{Uniform control in the stability ball}\label{app:ball}

Throughout this subsection, the \emph{gate} of hidden unit $r$ on input
$x_i$ at weights $Z=(z_1,\ldots,z_m)$ is the indicator
$\mathds1\{z_r^\top x_i\ge0\}$, and the gate \emph{flips} between $Z$ and
$Z'$ if $\mathds1\{z_r^\top x_i\ge0\}\neq\mathds1\{z_r'^\top x_i\ge0\}$.
These indicators are exactly the ones appearing in \eqref{eq:muon_gradient}
and \eqref{eq:muon_jacobian}.

\begin{lemma}[Uniform NTK and ReLU remainder]\label{lem:muon_ball}
On $\mathcal E$, for every $W,V\in\mathbb B$ and every matrix $O$,
\begin{align}
 \lambda_{\min}(H(W))&\ge\frac{\alpha^2\lambda_0}{2},
 \label{eq:muon_kernel_lower}\\
 \|J(W)-J(W_0)\|_F&\le\frac{\kappa\alpha\sqrt{\lambda_0}}{18M},
 \label{eq:muon_jacobian_close}\\
 \|(J(W)-J(V))\operatorname{vec}(O)\|_2
 &\le\frac{\kappa\alpha\sqrt{\lambda_0}}{18M}\,\|O\|_2,
 \label{eq:muon_jacobian_action}\\
 \|u(V)-u(W)-J(W)\operatorname{vec}(V-W)\|_2
 &\le\frac{\kappa\alpha\sqrt{\lambda_0}}{18M}\,\|V-W\|_2.
 \label{eq:muon_remainder}
\end{align}
The ball is measured in Frobenius norm and the last two bounds in operator
norm, because one step moves the weights by $\eta\|O_t\|_2\le\eta M$.
\end{lemma}

\begin{proof}[Proof of Lemma~\ref{lem:muon_ball}]
Fix an initialization in $\mathcal E$ and arbitrary
$W=(w_1,\ldots,w_m)$ and $V=(v_1,\ldots,v_m)$ in $\mathbb B$. Part 1
counts the gates that can flip; Part 2 turns that count into the Jacobian
bound \eqref{eq:muon_jacobian_close} and the kernel bound
\eqref{eq:muon_kernel_lower}; Parts 3 and 4 turn it into the remainder bound
\eqref{eq:muon_remainder} and the bound \eqref{eq:muon_jacobian_action}.

\noindent\textbf{Part 1.}
For $Z\in\{W,V\}$ and each $i$, define
\begin{equation}
 \mathcal S_i:=\Bigl\{r:|w_r(0)^\top x_i|\le\frac{\delta\kappa^2\lambda_0}{1620M^2n}\Bigr\},
 \qquad
 \mathcal I(Z):=\Bigl\{r:\|z_r-w_r(0)\|_2>\frac{\delta\kappa^2\lambda_0}{1620M^2n}\Bigr\}.
 \label{eq:appendix_margin_sets}
\end{equation}
If $r\notin\mathcal S_i\cup\mathcal I(Z)$, then, using $\|x_i\|_2=1$,
\begin{equation}
 |(z_r-w_r(0))^\top x_i|
 \le\|z_r-w_r(0)\|_2
 \le\frac{\delta\kappa^2\lambda_0}{1620M^2n}<|w_r(0)^\top x_i|,
 \label{eq:appendix_no_gate_change}
\end{equation}
so $z_r^\top x_i$ and $w_r(0)^\top x_i$ have the same nonzero sign and the
gate of unit $r$ on $x_i$ does not flip between $W_0$ and $Z$. Hence the set
\begin{equation}
 \mathcal A_i(W,V):=\{r:\text{the gate of $r$ on $x_i$ flips between
 $W_0$ and $W$, or between $W_0$ and $V$}\}
 \label{eq:appendix_bad_gate_set}
\end{equation}
satisfies
$\mathcal A_i(W,V)\subseteq\mathcal S_i\cup\mathcal I(W)\cup\mathcal I(V)$.
Each $r\in\mathcal I(Z)$ contributes more than
$(\delta\kappa^2\lambda_0/(1620M^2n))^2$ to $\|Z-W_0\|_F^2$, and $Z\in\mathbb B$,
so
\begin{equation}
 |\mathcal I(Z)|
 \le\Bigl(\frac{18M\sqrt n\,R_0}{\kappa\alpha\sqrt{\lambda_0}}\Bigr)^2
 \Bigl(\frac{1620M^2n}{\delta\kappa^2\lambda_0}\Bigr)^2
 \le\frac{8.51\cdot10^8\,M^6n^3R_0^2}{\kappa^6\alpha^2\lambda_0^3\delta^2}.
 \label{eq:appendix_large_move_count}
\end{equation}
Therefore
\begin{align}
 \frac1m\sum_{i=1}^n|\mathcal A_i(W,V)|
 &\le\frac1m\sum_{i=1}^n|\mathcal S_i|
   +\frac nm\bigl(|\mathcal I(W)|+|\mathcal I(V)|\bigr)
   \why{inclusion above}\nonumber\\
 &\le\frac{\kappa^2\lambda_0}{648M^2}
   +\frac{1.71\cdot10^9\,M^6n^4R_0^2}{m\,\kappa^6\alpha^2\lambda_0^3\delta^2}
   \why{on $\mathcal E_3$; by \eqref{eq:appendix_large_move_count}}\nonumber\\
 &\le\frac{\kappa^2\lambda_0}{324M^2}
   \why{width \eqref{eq:muon_width}, $R_0^2\le\tfrac{3\alpha^2}{\delta}\bigl(n+\tfrac{\|y\|_2^2}{\alpha^2}\bigr)$}.
 \label{eq:appendix_bad_gate_average}
\end{align}
The last inequality uses \eqref{eq:muon_width} together with
$648\cdot3\cdot1.71\cdot10^9<3.4\cdot10^{12}\le5\cdot10^{12}$.

\noindent\textbf{Part 2.}
Take $V=W_0$. A block of $J(W)-J(W_0)$ is nonzero only when the gate flips,
and then has squared norm $\alpha^2/m$. Hence
\begin{equation}
 \|J(W)-J(W_0)\|_F^2
 \le\frac{\alpha^2}{m}\sum_{i=1}^n|\mathcal A_i(W,W_0)|
 \le\frac{\kappa^2\alpha^2\lambda_0}{324M^2}
 \label{eq:appendix_jacobian_difference}
\end{equation}
by \eqref{eq:appendix_bad_gate_average}, which is \eqref{eq:muon_jacobian_close}.
For every $q\in\mathbb R^n$,
\begin{align}
 \|J(W)^\top q\|_2
 &\ge\|J(W_0)^\top q\|_2-\|J(W)-J(W_0)\|_2\|q\|_2
   \why{reverse triangle inequality}\nonumber\\
  &\ge\sqrt{q^\top H(W_0)q}-\frac{\kappa\alpha\sqrt{\lambda_0}}{18M}\|q\|_2
   \why{by \eqref{eq:appendix_jacobian_difference}}\nonumber\\
 &\ge\Bigl(\frac{\sqrt3}{2}-\frac1{18}\Bigr)\alpha\sqrt{\lambda_0}\,\|q\|_2
   \why{on $\mathcal E_2$; $\kappa\le M$}.
 \label{eq:appendix_ball_singular_value}
\end{align}
Here $\mathcal E_2$ gives $\lambda_{\min}(H(W_0))\ge\tfrac34\alpha^2\lambda_0$.
Since $(\sqrt3/2-1/18)^2>1/2$, squaring
\eqref{eq:appendix_ball_singular_value} gives \eqref{eq:muon_kernel_lower}.

\noindent\textbf{Part 3.}
Set
\begin{equation}
 \xi:=u(V)-u(W)-J(W)\operatorname{vec}(V-W),
 \qquad
 \xi_i=\frac1{\sqrt m}\sum_{r=1}^m a_r\,\rho\bigl(w_r^\top x_i,(v_r-w_r)^\top x_i\bigr),
 \label{eq:appendix_remainder_definition}
\end{equation}
where $\rho(z,h):=\sigma(z+h)-\sigma(z)-\mathds1\{z\ge0\}h$. The scalar
remainder satisfies $|\rho(z,h)|\le|h|$, and $\rho(z,h)=0$ unless the gates
at $z$ and $z+h$ differ. If neither $W$ nor $V$ flips a gate relative to
$W_0$, the gates at $W$ and $V$ agree, so only $r\in\mathcal A_i(W,V)$
contributes. Therefore
\begin{align}
 |\xi_i|
 &\le\frac{\alpha}{\sqrt m}\sum_{r\in\mathcal A_i(W,V)}|(v_r-w_r)^\top x_i|
   \why{$|a_r|=\alpha$, $|\rho(z,h)|\le|h|$}\nonumber\\
 &\le\alpha\sqrt{\frac{|\mathcal A_i(W,V)|}{m}}\,\|(V-W)^\top x_i\|_2
   \why{Cauchy--Schwarz over $r$}\nonumber\\
 &\le\alpha\sqrt{\frac{|\mathcal A_i(W,V)|}{m}}\,\|V-W\|_2
   \why{$\|x_i\|_2=1$}.
 \label{eq:appendix_coordinate_remainder}
\end{align}
Squaring and summing over $i$,
\begin{align}
 \|\xi\|_2^2
 &\le\alpha^2\|V-W\|_2^2\cdot\frac1m\sum_{i=1}^n|\mathcal A_i(W,V)|
   \nonumber\\
 &\le\frac{\kappa^2\alpha^2\lambda_0}{324M^2}\,\|V-W\|_2^2
   \why{by \eqref{eq:appendix_bad_gate_average}},
 \label{eq:appendix_vector_remainder}
\end{align}
which is \eqref{eq:muon_remainder}.

\noindent\textbf{Part 4.}
Write $o_r$ for the $r$-th column of $O$. The $i$-th entry of
$(J(W)-J(V))\operatorname{vec}(O)$ is
$\frac1{\sqrt m}\sum_ra_r\bigl(\mathds1\{w_r^\top x_i\ge0\}-\mathds1\{v_r^\top x_i\ge0\}\bigr)o_r^\top x_i$,
and as in Part 3 only $r\in\mathcal A_i(W,V)$ contribute. Hence
\begin{align}
 \|(J(W)-J(V))\operatorname{vec}(O)\|_2^2
 &\le\sum_{i=1}^n\alpha^2\frac{|\mathcal A_i(W,V)|}{m}\,\|O^\top x_i\|_2^2
   \why{$|a_r|=\alpha$, Cauchy--Schwarz over $r$}\nonumber\\
 &\le\frac{\kappa^2\alpha^2\lambda_0}{324M^2}\,\|O\|_2^2
   \why{$\|x_i\|_2=1$, \eqref{eq:appendix_bad_gate_average}},
 \label{eq:appendix_jacobian_action}
\end{align}
which is \eqref{eq:muon_jacobian_action}.
\end{proof}

\subsection{Newton--Schulz steps}\label{app:ns}

\begin{lemma}[Newton--Schulz matrix bounds]\label{lem:ns5}
Under Assumption~\ref{ass:ns}, every matrix $B$ whose column space lies in
$\operatorname{span}\{x_1,\ldots,x_n\}$ satisfies
\[
 \langle B,\mathrm{NS}_K(B)\rangle\ge\kappa\|B\|_F,\qquad
 \|\mathrm{NS}_K(B)\|_2\le M,\qquad
 \|\mathrm{NS}_K(B)\|_F\le M\sqrt n.
\]
\end{lemma}

\begin{proof}[Proof of Lemma~\ref{lem:ns5}]
By Assumption~\ref{ass:ns},
\begin{equation}
 \kappa x\le\Phi(x)\le M,\qquad x\in[0,1].
 \label{eq:appendix_scalar_ns_bounds}
\end{equation}
The case $B=0$ holds by the convention $\mathrm{NS}_K(0)=0$. Otherwise let
$B=U\Sigma V^\top$ be a compact SVD with positive singular values
$s_1,\ldots,s_k$. Each step of \eqref{eq:ns_iteration} keeps $U,V$ and
applies $\phi_k$ to the diagonal, so
\begin{equation}
 \mathrm{NS}_K(B)
 =U\operatorname{diag}\Bigl(
 \Phi\Bigl(\frac{s_1}{\|B\|_F}\Bigr),\ldots,
 \Phi\Bigl(\frac{s_k}{\|B\|_F}\Bigr)\Bigr)V^\top,
 \qquad \frac{s_j}{\|B\|_F}\in(0,1].
 \label{eq:appendix_ns_svd}
\end{equation}
By \eqref{eq:appendix_scalar_ns_bounds} the diagonal entries lie in $[0,M]$,
so $\|\mathrm{NS}_K(B)\|_2\le M$. The column space of $B$ lies in
$\operatorname{span}\{x_1,\ldots,x_n\}$, so $k\le n$ and
$\|\mathrm{NS}_K(B)\|_F\le M\sqrt k\le M\sqrt n$. Finally,
\begin{align}
 \langle B,\mathrm{NS}_K(B)\rangle
 &=\sum_{j=1}^k s_j\,\Phi\Bigl(\frac{s_j}{\|B\|_F}\Bigr)
   \why{shared singular vectors in \eqref{eq:appendix_ns_svd}}\nonumber\\
 &\ge\frac{\kappa}{\|B\|_F}\sum_{j=1}^k s_j^2
   \why{lower bound in \eqref{eq:appendix_scalar_ns_bounds}}\nonumber\\
 &=\kappa\|B\|_F
   \why{$\sum_j s_j^2=\|B\|_F^2$}.
 \label{eq:appendix_ns_alignment}
\end{align}
\end{proof}

\subsection{One-step descent}\label{app:step}

\begin{proof}[Proof of Lemma~\ref{lem:muon_step}]
Every buffer column lies in $\operatorname{span}\{x_1,\ldots,x_n\}$, so
Lemma~\ref{lem:ns5} applies to $B_t$, and $W_t-W_{t-1}=-\eta O_t$. Part 1
bounds how far the predictions move; Part 2 expands the new residual and
proves \eqref{eq:muon_squared_descent}; Part 3 converts it into
\eqref{eq:muon_linear_descent}.

\noindent\textbf{Part 1.}
For each $i$,
\begin{align}
 |u(W_t)_i-u(W_{t-1})_i|
 &\le\frac{\alpha\eta}{\sqrt m}\sum_{r=1}^m|(O_t)_{\cdot,r}^\top x_i|
   \why{$\sigma$ is $1$-Lipschitz, $|a_r|=\alpha$}\nonumber\\
 &\le\alpha\eta\,\|O_t^\top x_i\|_2
   \why{Cauchy--Schwarz over $r$}\nonumber\\
 &\le M\alpha\eta
   \why{$\|x_i\|_2=1$, Lemma~\ref{lem:ns5}},
 \label{eq:appendix_prediction_coordinate}
\end{align}
so $\|u(W_t)-u(W_{t-1})\|_2\le\sqrt n\,M\alpha\eta$. This bound holds at every
step, whether or not the iterates lie in $\mathbb B$.

\noindent\textbf{Part 2.}
By hypothesis \eqref{eq:muon_stale_condition} and the kernel bound
\eqref{eq:muon_kernel_lower},
\begin{equation}
 \langle G_t,O_t\rangle\ge\frac\kappa2\|G_t\|_F
 =\frac\kappa2\bigl(e_{t-1}^\top H(W_{t-1})e_{t-1}\bigr)^{1/2}
 \ge\frac{\kappa\alpha\sqrt{\lambda_0}}{2\sqrt2}\,\|e_{t-1}\|_2.
 \label{eq:appendix_current_alignment}
\end{equation}
Let $\xi:=u(W_t)-u(W_{t-1})+\eta J(W_{t-1})\operatorname{vec}(O_t)$ be the
linearization remainder of the step. Both iterates lie in $\mathbb B$, so
\eqref{eq:muon_remainder} and Lemma~\ref{lem:ns5} give
$\|\xi\|_2\le\frac{\kappa\alpha\sqrt{\lambda_0}}{18M}\cdot M\eta
=\frac{\kappa\alpha\eta\sqrt{\lambda_0}}{18}$. Moreover,
\begin{equation}
 -\eta\,e_{t-1}^\top J(W_{t-1})\operatorname{vec}(O_t)
 =-\eta\,(J(W_{t-1})^\top e_{t-1})^\top\operatorname{vec}(O_t)
 =\eta\langle G_t,O_t\rangle.
 \label{eq:appendix_linear_gain_identity}
\end{equation}
Since $e_t=e_{t-1}-(u(W_t)-u(W_{t-1}))$,
\begin{align}
 \|e_t\|_2^2
 &=\|e_{t-1}\|_2^2-2\eta\langle G_t,O_t\rangle-2e_{t-1}^\top\xi
   +\|u(W_t)-u(W_{t-1})\|_2^2
   \why{expand, \eqref{eq:appendix_linear_gain_identity}}\nonumber\\
 &\le\|e_{t-1}\|_2^2
   -\Bigl(\frac1{\sqrt2}-\frac2{18}\Bigr)\kappa\alpha\eta\sqrt{\lambda_0}\,\|e_{t-1}\|_2
   +nM^2\alpha^2\eta^2
   \why{\eqref{eq:appendix_current_alignment}, Part~1, Cauchy--Schwarz}\nonumber\\
 &\le\|e_{t-1}\|_2^2
   -\frac{\kappa\alpha\eta\sqrt{\lambda_0}}{4}\|e_{t-1}\|_2
   +nM^2\alpha^2\eta^2
   \why{$\tfrac1{\sqrt2}-\tfrac2{18}>\tfrac14$}.
 \label{eq:appendix_squared_descent}
\end{align}
This is \eqref{eq:muon_squared_descent}.

\noindent\textbf{Part 3.}
Assume \eqref{eq:muon_local_stepsize}; since $\eta>0$ it forces
$\|e_{t-1}\|_2>0$. Multiplying \eqref{eq:muon_local_stepsize} by
$nM^2\alpha^2\eta$ gives
$nM^2\alpha^2\eta^2\le\frac{\kappa\alpha\eta\sqrt{\lambda_0}}{8}\|e_{t-1}\|_2$, so
\begin{align}
 \|e_t\|_2^2
 &\le\|e_{t-1}\|_2^2-\frac{\kappa\alpha\eta\sqrt{\lambda_0}}{8}\|e_{t-1}\|_2
   \why{by \eqref{eq:appendix_squared_descent}}\nonumber\\
 &=\|e_{t-1}\|_2^2\Bigl(1-\frac{\kappa\alpha\eta\sqrt{\lambda_0}}{8\|e_{t-1}\|_2}\Bigr).
 \label{eq:appendix_before_square_root}
\end{align}
The left side is nonnegative, so the bracket lies in $[0,1]$, and
\begin{align}
 \|e_t\|_2
 &\le\|e_{t-1}\|_2\Bigl(1-\frac{\kappa\alpha\eta\sqrt{\lambda_0}}{16\|e_{t-1}\|_2}\Bigr)
   \why{$\sqrt{1-z}\le1-z/2$ on $[0,1]$}\nonumber\\
 &=\|e_{t-1}\|_2-\frac{\kappa\alpha\eta\sqrt{\lambda_0}}{16},
 \label{eq:appendix_linear_descent}
\end{align}
which is \eqref{eq:muon_linear_descent}.
\end{proof}

\subsection{Momentum}\label{app:momentum}

\begin{proof}[Proof of Lemma~\ref{lem:muon_momentum}]
Write $J_k:=J(W_k)$. Unrolling $B_t=\mu B_{t-1}+G_t$ down to $B_0=0$ splits the
buffer into the current gradient and the change of the older ones,
\begin{equation}
 B_t=S\,G_t+D,\qquad
 S:=\sum_{j=0}^{t-1}\mu^j\ge1,\qquad
 D:=\sum_{j=1}^{t-1}\mu^j\bigl(G_{t-j}-G_t\bigr).
 \label{eq:appendix_buffer_split}
\end{equation}
Part 1 reduces the claim to a bound on $D$, Part 2 bounds one term of $D$, and
Part 3 adds the terms up.

\noindent\textbf{Part 1.}
The column space of $B_t$ lies in $\operatorname{span}\{x_1,\ldots,x_n\}$, so
\begin{align}
 S\langle G_t,O_t\rangle
 &=\langle B_t,O_t\rangle-\langle D,O_t\rangle
   \why{by \eqref{eq:appendix_buffer_split}}\nonumber\\
 &\ge\kappa\|B_t\|_F-|\langle D,O_t\rangle|
   \why{Lemma~\ref{lem:ns5}}\nonumber\\
 &\ge\kappa S\|G_t\|_F-\kappa\|D\|_F-|\langle D,O_t\rangle|
   \why{triangle inequality}.
 \label{eq:appendix_momentum_reduction}
\end{align}
It therefore suffices to show
\begin{equation}
 \kappa\|D\|_F+|\langle D,O_t\rangle|\le\frac\kappa2\,S\|G_t\|_F.
 \label{eq:appendix_momentum_target}
\end{equation}

\noindent\textbf{Part 2.}
Fix $1\le j\le t-1$ and put $s=t-j$. By the gradient identity,
$\operatorname{vec}(G_s-G_t)
=-J_{s-1}^\top(e_{s-1}-e_{t-1})-(J_{s-1}-J_{t-1})^\top e_{t-1}$. Two facts
control the first term. By Part~1 of the proof of Lemma~\ref{lem:muon_step},
the residual moves by at most $\sqrt n\,M\alpha\eta$ per step, so
$\|e_{s-1}-e_{t-1}\|_2\le j\sqrt n\,M\alpha\eta$. And for every $W$ and every
matrix $O$, each entry of $J(W)\operatorname{vec}(O)$ is at most
$\alpha\|O^\top x_i\|_2\le\alpha\|O\|_2$ in magnitude, by Cauchy--Schwarz over
the hidden units, so $\|J(W)\operatorname{vec}(O)\|_2\le\alpha\sqrt n\,\|O\|_2$
and $\|J(W)\|_2\le\alpha\sqrt n$. Since $W_{s-1},W_{t-1}\in\mathbb B$,
\begin{align}
 \|G_s-G_t\|_F
 &\le\|J_{s-1}\|_2\|e_{s-1}-e_{t-1}\|_2+\|J_{s-1}-J_{t-1}\|_2\|e_{t-1}\|_2
   \why{triangle inequality}\nonumber\\
 &\le jnM\alpha^2\eta+\frac{2\kappa\alpha\sqrt{\lambda_0}}{18M}\|e_{t-1}\|_2
   \why{facts above, \eqref{eq:muon_jacobian_close} twice},
 \label{eq:appendix_gradient_change_norm}\\
 |\langle G_s-G_t,O_t\rangle|
 &\le\|e_{s-1}-e_{t-1}\|_2\|J_{s-1}\operatorname{vec}(O_t)\|_2\nonumber\\
 &\quad+\|e_{t-1}\|_2\|(J_{s-1}-J_{t-1})\operatorname{vec}(O_t)\|_2
   \why{Cauchy--Schwarz}\nonumber\\
 &\le jnM^2\alpha^2\eta+\frac{\kappa\alpha\sqrt{\lambda_0}}{18}\|e_{t-1}\|_2
   \why{facts above, \eqref{eq:muon_jacobian_action}, $\|O_t\|_2\le M$}.
 \label{eq:appendix_gradient_change_inner}
\end{align}
Since $\kappa\le M$, the two bounds combine to
\begin{equation}
 \kappa\|G_s-G_t\|_F+|\langle G_s-G_t,O_t\rangle|
 \le2jnM^2\alpha^2\eta+\frac{3\kappa\alpha\sqrt{\lambda_0}}{18}\|e_{t-1}\|_2.
 \label{eq:appendix_gradient_change}
\end{equation}

\noindent\textbf{Part 3.}
Exchanging the order of summation and using $\sum_{k=0}^{t-1-i}\mu^k\le S$,
\begin{equation}
 \sum_{j=1}^{t-1}j\mu^j
 =\sum_{i=1}^{t-1}\sum_{j=i}^{t-1}\mu^j
 =\sum_{i=1}^{t-1}\mu^i\sum_{k=0}^{t-1-i}\mu^k
 \le S\sum_{i=1}^{t-1}\mu^i
 \le\frac{S\mu}{1-\mu},
 \label{eq:appendix_weighted_j}
\end{equation}
and also $\sum_{j=1}^{t-1}\mu^j\le S$. Summing \eqref{eq:appendix_gradient_change} with
weights $\mu^j$,
\begin{align}
 \kappa\|D\|_F+|\langle D,O_t\rangle|
 &\le S\Bigl(\frac{2\mu}{1-\mu}\,nM^2\alpha^2\eta
   +\frac{3\kappa\alpha\sqrt{\lambda_0}}{18}\|e_{t-1}\|_2\Bigr)
   \why{triangle inequality, weight sums}\nonumber\\
 &\le S\Bigl(\frac\kappa4\,\alpha\sqrt{\lambda_0\varepsilon}
   +\frac{3\kappa\alpha\sqrt{\lambda_0}}{18}\|e_{t-1}\|_2\Bigr)
   \why{bound on $\eta$, $\mu\le1$}\nonumber\\
 &\le\frac\kappa2\,S\|G_t\|_F
   \why{see below}.
 \label{eq:appendix_momentum_sum}
\end{align}
The last line uses the kernel bound \eqref{eq:muon_kernel_lower} at $W_{t-1}$
and $\|e_{t-1}\|_2>\sqrt{2\varepsilon}$, which give
$\|G_t\|_F\ge\alpha\sqrt{\lambda_0/2}\,\|e_{t-1}\|_2>\alpha\sqrt{\lambda_0\varepsilon}$,
together with $\frac3{18}\le\frac1{4\sqrt2}$; each of the two terms is then at
most $\frac\kappa4\|G_t\|_F$. This is \eqref{eq:appendix_momentum_target}, and
with \eqref{eq:appendix_momentum_reduction} it gives
$\langle G_t,O_t\rangle\ge\kappa\|G_t\|_F-\frac\kappa2\|G_t\|_F$. For $t=1$
the sums are empty, $D=0$, and the claim is Lemma~\ref{lem:ns5} directly.
\end{proof}

\subsection{Constants of Muon's coefficients and of the classical polynomials}
\label{app:jordan_coefficients}

The classical schedules used below are defined by the Taylor truncations
\begin{equation}
 p^{(d)}(\lambda)=\sum_{s=0}^{d}\frac{(2s)!}{4^s(s!)^2}(1-\lambda)^s,
 \qquad \phi^{(d)}(x)=x p^{(d)}(x^2).
 \label{eq:classical_ns}
\end{equation}
Here $d=1$ gives the cubic scalar iteration and $d=2$ the quintic.
The tuned map instead uses the coefficients specified in the main text.

\begin{proposition}\label{prop:jordan_constants}
For $p(\lambda)=a+b\lambda+c\lambda^2$ with $(a,b,c)=(3.4445,-4.7750,2.0315)$,
$\phi(x)=x\,p(x^2)$ and $\Phi=\phi^{\circ5}$,
\begin{equation}
 \Phi(x)\le\frac54\quad\text{and}\quad\Phi(x)\ge\frac23x\quad\text{for }x\in[0,1],
 \qquad
 \Phi(x)\ge0.68\quad\text{for }x\in\Bigl[\frac1{400},1\Bigr].
 \label{eq:appendix_jordan_claims}
\end{equation}
\end{proposition}

\begin{proof}
The proof uses four facts about the single quintic $\phi$, each an evaluation
of a polynomial. Let
$x_a\approx0.55453$ and $x_b\approx1.05014$ be the positive roots of
$\phi'(x)=a+3bx^2+5cx^4$.
\begin{enumerate}
 \item[(F1)] $p(z)\ge a-b^2/(4c)=0.6386\ldots>0$ for all $z$, and $p$ decreases
 on $[0,-b/(2c)]=[0,1.1752\ldots]$.
 \item[(F2)] $\phi$ increases on $[0,x_a]$, decreases on $[x_a,x_b]$ and
 increases on $[x_b,\frac54]$, with $\phi(x_a)=1.2023\ldots$,
 $\phi(x_b)=0.6818\ldots$ and $\phi(\frac54)=1.1790\ldots$.
 \item[(F3)] $p(x_a^2)=2.1682\ldots$.
 \item[(F4)] The first five iterates of $\frac1{400}$ under $\phi$ are
 $0.0086\ldots$, $0.0296\ldots$, $0.1020\ldots$, $0.3464\ldots$ and
 $1.0048\ldots$.
\end{enumerate}
By (F1) and (F2), $\phi$ maps $[0,\frac54]$ into $[0,1.2024]$ and maps
$[x_a,\frac54]$ into $[0.6818,1.2024]$. Fix $x\in(0,1]$ and let
$x_k=\phi^{\circ k}(x)$, so every $x_k$ lies in $[0,1.2024]$ and in particular
$\Phi(x)=x_5\le\frac54$.

\noindent\textbf{Part 1.}
Suppose $x_k\ge x_a$ for some $k\le4$. Then $x_k\in[x_a,\frac54]$, every later
iterate stays in $[0.6818,1.2024]\subseteq[x_a,\frac54]$, and so
$\Phi(x)\ge0.6818$. In particular $\Phi(x)\ge0.68$ and
$\Phi(x)\ge\frac23x$, since $x\le1$.

\noindent\textbf{Part 2.}
Suppose $x_k<x_a$ for every $k\le4$. Since $p$ decreases on $[0,x_a^2]$,
\begin{align}
 \frac{\Phi(x)}x
 &=\prod_{k=0}^4p(x_k^2)
   \why{$\phi(x)=x\,p(x^2)$}\nonumber\\
 &\ge p(x_a^2)^5
   \why{(F1), $x_k^2<x_a^2$}\nonumber\\
 &\ge1
   \why{(F3)},
 \label{eq:appendix_small_orbit}
\end{align}
so $\Phi(x)\ge\frac23x$. If moreover $x\ge\frac1{400}$, then, since $\phi$
increases on $[0,x_a]$, induction on $k$ gives
$x_k\ge\phi^{\circ k}(\frac1{400})$ for $k\le4$, and so
$\Phi(x)=\phi(x_4)\ge\phi\bigl(\phi^{\circ4}(\frac1{400})\bigr)=1.0048\ldots\ge0.68$
by (F4).

Parts~1 and~2 cover all $x\in(0,1]$ and prove \eqref{eq:appendix_jordan_claims};
$x=0$ is trivial.
\end{proof}

The first two claims give Assumption~\ref{ass:ns} for Muon's coefficients
with $\kappa=\frac23$ and $M=\frac54$, the values used in
Theorem~\ref{thm:main_muon}.
The third claim also yields a nuclear-norm alignment bound: if every nonzero singular value of $B$ is
at least $\|B\|_F/400$, then by \eqref{eq:appendix_ns_svd}
$\langle B,\mathrm{NS}_5(B)\rangle=\sum_js_j\Phi(s_j/\|B\|_F)\ge0.68\,\|B\|_*$.

\begin{proposition}[Classical Newton--Schulz polynomials]\label{prop:classical_ns}
Let $\Phi$ be the composition of $K$ classical polynomials
\eqref{eq:classical_ns} of degrees $d_0,\ldots,d_{K-1}\ge1$. For $x\in[0,1]$,
\begin{equation}
 x\le\Phi(x)\le1
 \qquad\text{and}\qquad
 1-\Phi(x)^2\le(1-x^2)^{\prod_k(d_k+1)}.
 \label{eq:appendix_classical_claims}
\end{equation}
Hence Assumption~\ref{ass:ns} holds with $\kappa=M=1$ for every degree and
every number of steps. For five steps of the same degree, $\Phi(x)\ge0.68$ on
$[\frac17,1]$ when $d=1$ and on $[\frac1{19},1]$ when $d=2$.
\end{proposition}

\begin{proof}
\noindent\textbf{Part 1.}
For $\lambda\in(0,1]$, the Taylor series
$\lambda^{-1/2}=\sum_{s\ge0}\frac{(2s)!}{4^s(s!)^2}(1-\lambda)^s$ converges and
has nonnegative terms, and $p^{(d)}$ is its truncation after the term $s=d$.
Hence $1\le p^{(d)}(\lambda)\le\lambda^{-1/2}$, and with $\lambda=x^2$,
\begin{equation}
 x\le x\,p^{(d)}(x^2)=\phi(x)\le1,\qquad x\in(0,1].
 \label{eq:appendix_classical_step}
\end{equation}
So each step maps $[0,1]$ into itself without decreasing its argument, and
composing $K$ steps gives $x\le\Phi(x)\le1$.

\noindent\textbf{Part 2.}
For a classical polynomial of degree $d$ and $u\in[0,1]$,
$1-(1-u)\,p^{(d)}(1-u)^2\le u^{d+1}$ \citep[Lemma~3]{kim2026muonns}. With
$u=1-x^2$ this reads $1-\phi(x)^2\le(1-x^2)^{d+1}$. By Part~1 every iterate
stays in $[0,1]$, so applying this bound at each of the $K$ steps gives the
second claim of \eqref{eq:appendix_classical_claims}.

The bound $\Phi(x)\ge\bigl(1-(1-x^2)^{\prod_k(d_k+1)}\bigr)^{1/2}$ increases in
$x$. For $K=5$ and $d=1$ the exponent is $32$ and
$(48/49)^{32}<0.52$, so $\Phi\ge0.69$ on $[\frac17,1]$. For $d=2$ the
exponent is $243$ and $(360/361)^{243}<0.51$, so $\Phi\ge0.70$ on
$[\frac1{19},1]$.
\end{proof}

\section{Derivation of the infinite-width kernel}\label{app:kernel_derivation}

Let $w\sim\mathcal N(0,I_d)$. The pair $(w^\top x_i,w^\top x_j)$ is a centered
bivariate Gaussian with unit variances and correlation $x_i^\top x_j$. The two
half-space conditions intersect in an angular wedge of length
$\pi-\arccos(x_i^\top x_j)$, so
\begin{equation}
 \Pr(w^\top x_i\ge0,\ w^\top x_j\ge0)
 =\frac{\pi-\arccos(x_i^\top x_j)}{2\pi}.
 \label{eq:appendix_orthant_probability}
\end{equation}
Multiplying by $x_i^\top x_j$ gives \eqref{eq:muon_hinf}.

\section{Polar-Factor Optimality and Finite-Step Alignment}
\label{sec:assumption_reading}
\label{app:polar_optimality}

The exact polar factor maximizes alignment with a given buffer per unit
spectral norm. This property concerns the buffer; transfer to the current
gradient additionally requires Lemma~\ref{lem:muon_momentum}.

\begin{proposition}[Optimal buffer alignment per unit spectral norm]
\label{prop:polar_optimal}
Let $B\ne0$ have compact singular value decomposition $B=U\Sigma V^\top$.
For every nonzero matrix $O$ of the same dimensions,
\begin{equation}
 \frac{\langle B,O\rangle}{\|O\|_2}\le\|B\|_*,
 \label{eq:polar_optimal}
\end{equation}
with equality for $O=UV^\top$.
\end{proposition}
\begin{proof}
The trace inequality gives
$\langle B,O\rangle\le\sum_j s_j(B)s_j(O)
\le\|O\|_2\sum_j s_j(B)=\|O\|_2\|B\|_*$.
For $O=UV^\top$, $\|O\|_2=1$ and
$\langle B,O\rangle=\operatorname{tr}(\Sigma)=\|B\|_*$.
\end{proof}

Under Assumption~\ref{ass:ns}, the finite map satisfies
\[
 \frac{\langle B,\mathrm{NS}_K(B)\rangle}{\|\mathrm{NS}_K(B)\|_2}
 \ge\frac{\kappa}{M}\|B\|_F.
\]
This bound suffices for the training analysis without establishing a
uniform, rank-independent fraction of the polar optimum. If every nonzero
singular value of $B$ is at least $\|B\|_F/400$, then
Proposition~\ref{prop:jordan_constants} additionally gives
$\langle B,\mathrm{NS}_5(B)\rangle\ge0.68\|B\|_*$.
Together with $M\le5/4$, this implies an alignment-per-norm ratio of at
least $0.544\|B\|_*$. This stronger comparison requires the stated
spectral condition and is not assumed in the training theorem.

\section{Experimental Details and Additional Diagnostics}
\label{app:experimental_details}

This section collects the protocols and supporting diagnostics accompanying
Section~\ref{sec:experiments}. The finite-map comparison is interpreted
using Proposition~\ref{prop:polar_optimal}.

\subsection{Data, Initialization, and Numerical Settings}
\label{app:experiment_setup}

Inputs are sampled from a standard Gaussian distribution
in $\mathbb R^d$ and normalized individually to unit norm.
For the saved experiments, we use one dataset with seed 0
and an independent two-layer ReLU teacher with seed 1,
width 1000, and output weights sampled from $\{\pm1\}$.
The teacher output uses the $1/\sqrt{1000}$ normalization
in \eqref{eq:muon_network}. The teacher and labels remain
fixed across student initializations and widths.
Unless otherwise specified, $n=20$, $d=100$, $m=1000$,
$\alpha=1$, and $\delta=0.05$. Student initialization
and the unnormalized squared-loss convention follow
Section~\ref{sec:model}. The alignment and fixed-step
momentum experiments use student seeds 10, 11, and 12;
the width sweep uses seeds 10--14 at each width.

For the saved dataset, we compute $\lambda_0=0.335$
from \eqref{eq:muon_hinf} and $R_0=30.2$ from
\eqref{eq:muon_constants}. With $\kappa=2/3$,
$M=5/4$, and target $\varepsilon=10^{-2}$,
the prescribed learning rate is
$\eta=(1-\mu)\eta_0$, where
$\eta_0\approx1.54\times10^{-4}$.
The fixed-step comparison instead uses $\eta=\eta_0$
for every momentum value. The reference budget $T$
is computed from \eqref{eq:muon_budget}; practical
iteration caps are listed separately in
Table~\ref{tab:appendix_experiment_settings} and
are distinguished from $T$ in
Appendix~\ref{app:diagnostic_protocol}.
Figure~\ref{fig:scaling}(c,d) uses the saved seed-10
trajectories from the alignment experiment.

Figure~\ref{fig:scaling}(a,b) shows a separate accuracy
sweep on a different draw of the data, with $\lambda_0\approx0.34$,
$R_0\approx26$, and $\mu=0.95$.

The experiments use Python 3.14.6 and NumPy 2.5.3.
Newton--Schulz calculations use float64 arithmetic
and Frobenius normalization without a denominator
regularizer. Double precision reduces roundoff but
does not make the computation identical to the
exact-arithmetic map analyzed in the theorem.
The SVD routine, numerical-rank threshold, logging
schedule, and run horizons are specified in
Table~\ref{tab:appendix_experiment_settings}.

\begin{table}[t]
\centering
\caption{Experimental settings for the diagnostics in this appendix and
in Section~\ref{sec:experiments}. Learning rates use the measured
constants of the frozen draw, $\lambda_0=0.335$ and $R_0=30.2$.
Horizons are the iteration caps actually used, not the theoretical budget.}
\label{tab:appendix_experiment_settings}
\begin{tabular}{@{}p{0.34\linewidth}p{0.60\linewidth}@{}}
\toprule
Setting & Value or status\\
\midrule
Default data and student & $n=20$, $d=100$, $m=1000$, $\alpha=1$\\
Budget parameter & $\delta=0.05$\\
Accuracy panels & $\varepsilon\in\{10^{-1},10^{-2},10^{-3},10^{-4},10^{-5}\}$;
$\mu=0.95$; separate draw with $R_0\approx26$\\
Displayed momentum values & $\mu\in\{0,0.5,0.9,0.95,0.99\}$\\
Fixed-step comparison & Target $10^{-2}$; shared step $\eta=\eta_0=1.54\times10^{-4}$; horizon $50{,}000$\\
Teacher specification & Width $1000$, output weights $\pm 1$, seed $1$, labels scaled by $m^{-1/2}$\\
Repetitions & One dataset (seed $0$) and one teacher. Alignment and fixed-step: student seeds $10,11,12$. Width: seeds $10$--$14$\\
Alignment study & Target $10^{-2}$; $\mu\in\{0,0.5,0.9,0.95,0.99\}$; $\eta=(1-\mu)\eta_0$; horizons $70{,}700$ to $7{,}070{,}000$\\
Width study & $m\in\{64,128,256,512,1000,2000\}$; target $10^{-2}$; $\mu=0.95$; horizon $1{,}414{,}000$\\
Logging & $A_t,D_t,C_t$ every step. Kernel and displacement every step in the width sweep. Buffers: $12$ geometric steps plus the last, seed $10$ only ($59$ buffers)\\
SVD computation & \texttt{numpy.linalg.svd}; keep $s>\max(d,m)\,\varepsilon_{\mathrm{float64}}\,s_{\max}$. Discarded nuclear fraction $\le 8\times10^{-15}$\\
Implementation & Python $3.14.6$, NumPy $2.5.3$, float64. No denominator regularizer\\
Variability summaries & Each seed has equal weight. Bands and error bars are mean $\pm 1$ standard deviation across seeds. Extrema are over every logged iteration\\
\bottomrule
\end{tabular}
\end{table}

\subsection{Stopping, Logging, and Alignment Conventions}
\label{app:diagnostic_protocol}

For each run, we define
$t_{\mathrm{hit}}=\min\{t:L(W_t)\le\varepsilon\}$,
with $t_{\mathrm{hit}}=0$ if the initial iterate
already satisfies the target. Training terminates
at target attainment or the practical horizon $H$
listed in Table~\ref{tab:appendix_experiment_settings}.
A run that does not reach the target within $H$
is recorded as unsuccessful within that horizon;
it is not assigned a hitting time of $H$.
The practical cap is distinct from the theorem's
reference budget. For the saved dataset with
$\mu=0.95$ and $\varepsilon=10^{-2}$,
$T\approx1.62\times10^8$, whereas the practical
horizon is $H=1{,}414{,}000$.
The parameter $\delta$ specifies the theoretical
reference calculation, not an empirical confidence
interval or a guaranteed failure probability at
the tested widths.

The alignment quantities $A_t$ and $D_t$ in
\eqref{eq:alignment_diagnostics} are recorded at
every update, including the update that first
reaches the target. The quantity $A_t$ is not a
cosine similarity because it is normalized only
by the gradient norm. The normalization
$B_t/s_t$, with $s_t=(1-\mu^t)/(1-\mu)$,
removes the accumulation factor from the momentum
buffer. Consequently, $D_t=0$ whenever all stored
gradients coincide; in particular, $D_t=0$ for
$\mu=0$. Normalized diagnostics are defined only
when their denominators are nonzero. A zero
denominator is treated as an undefined measurement
rather than assigned a numerical value.

Figure~\ref{fig:alignment} uses student seeds
10, 11, and 12. At each iteration, the plotted
mean includes only seeds that have not already
terminated; each contributing seed receives
equal weight. Shaded bands show one sample
standard deviation above and below the mean
when at least two seeds remain. Thus, the number
of contributing seeds decreases after the first
target crossing, and changes near the endpoints
can reflect both within-run dynamics and changes
in the set of contributing runs. The reported
minimum alignment and ranges of run-wise maximum
discrepancy are computed from individual runs,
rather than from these averaged curves.

We also record buffer-update alignment at every
update:
\begin{equation}
    C_t:=
    \frac{\langle B_t,O_t\rangle}{\|B_t\|_F}.
    \label{eq:empirical_buffer_alignment}
\end{equation}
The reference $C_t\ge\kappa=2/3$ follows from the
tuned Newton--Schulz map in exact arithmetic,
independently of network width. In contrast,
the current-gradient bound $A_t\ge\kappa/2$
additionally requires the trajectory assumptions
of Lemma~\ref{lem:muon_momentum}.
For the matrix-map comparison, buffers are saved
from student seed 10 at geometrically spaced
updates and at the final update, yielding 59
buffers across the tested momentum values.
These measurements characterize the saved buffers;
they do not constitute independent training runs.

The fixed-step comparison uses
$\eta=\eta_0\approx1.54\times10^{-4}$,
target $10^{-2}$, and horizon $50{,}000$ for
$\mu\in\{0,0.5,0.9,0.95,0.99\}$.
All three student seeds at each momentum reach
the target, with hitting times ranging from
14{,}680 to 25{,}444 updates.
These runs show that the prescribed
$(1-\mu)$ learning-rate scaling can be conservative
for the tested problem; they do not establish
a generally valid momentum-independent step size.

\subsection{Finite Newton--Schulz Maps and the Exact Polar Factor}
\label{app:ns_comparison}

We compare three transformations on identical nonzero
momentum buffers saved from tuned Muon trajectories:
the tuned five-step Newton--Schulz map, five classical
cubic Newton--Schulz iterations, and the exact polar
factor evaluated numerically by SVD. Both iterative
maps start from $X_0=B/\|B\|_F$. The classical
recurrence is
\begin{equation}
    X_{k+1}
    =\frac32X_k-\frac12X_kX_k^\top X_k,
    \qquad k=0,\ldots,4,
    \label{eq:classical_cubic_comparison}
\end{equation}
corresponding to $d=1$ in \eqref{eq:classical_ns}.
Although Proposition~\ref{prop:classical_ns} also
covers higher-degree classical polynomials, the
empirical comparison here uses only the cubic map.
Evaluating each transformation on the same buffers
isolates its effect from differences in training
trajectories.

We compute the SVD using \texttt{numpy.linalg.svd}
and retain singular values satisfying
$s_j>\max(d,m)\epsilon_{\mathrm{float64}}s_{\max}$,
where $\epsilon_{\mathrm{float64}}$ is machine
precision. The numerical polar update is
$U_rV_r^\top$, formed from the retained singular-vector
pairs. Across the 59 saved buffers, the discarded
singular values account for at most
$8\times10^{-15}$ of the nuclear norm.
Thus, truncation has negligible effect on the
alignment comparison at the reported precision.

For each buffer $B$ and transformed update $O$,
we evaluate
\begin{equation}
    a(B,O):=\frac{\langle B,O\rangle}{\|B\|_F},
    \qquad
    s(O):=\|O\|_2,
    \qquad
    q(B,O):=
    \frac{\langle B,O\rangle}
         {\|B\|_*\|O\|_2}.
    \label{eq:map_diagnostics}
\end{equation}
The quantities $a$ and $s$ examine the alignment
and norm bounds used in the training proof.
The ratio $q$ measures alignment per unit spectral
norm relative to the optimum in
Proposition~\ref{prop:polar_optimal}; the exact
polar factor attains $q=1$.
These quantities do not measure distance to the
polar factor.

Figure~\ref{fig:ns_comparison} displays the
measurements on the 59 buffers from student
seed 10. Tuned NS5 has
$a\in[0.697,2.690]$, with median $2.687$, and
$s\in[0.696,1.202]$. Thus, all sampled buffers
satisfy the analytical references $a\ge2/3$
and $s\le5/4$. Its normalized ratio lies in
$q\in[0.748,0.999]$, with median $0.764$ and
standard deviation $0.075$ across buffers.
Classical cubic NS5 has $s=1$ to the reported
precision and $q\in[0.860,0.998]$, with median
$0.860$ and standard deviation $0.043$.
The numerical polar update has $s=1$ and $q=1$
to the reported precision.

The classical cubic map therefore achieves a
higher median alignment-per-norm ratio on these
sampled buffers, while the tuned map retains
the properties sufficient for the training
guarantee. The medians and standard deviations
summarize correlated buffers from the saved
trajectories, not variability across independent
datasets. This comparison establishes neither
a ranking of training performance nor which map
more closely approximates the polar factor.

\subsection{Width, Displacement, and Kernel Diagnostics}
\label{app:width_details}

We vary the student width over
$m\in\{64,128,256,512,1000,2000\}$ while holding
the dataset, teacher, target loss, momentum, and
prescribed learning rate fixed. Each width uses
student seeds 10--14, giving 30 runs in total.
We set $\varepsilon=10^{-2}$, $\mu=0.95$, and
the practical horizon to $1{,}414{,}000$ updates.
For this dataset, $\lambda_0=0.335$ and $R_0=30.2$.
Kernel and displacement diagnostics are recorded
at every iteration through target attainment
or horizon exhaustion.

Alongside $K_t$ and $S_t$ in
\eqref{eq:width_diagnostics}, we measure
\begin{equation}
    P_t:=\|W_t-W_0\|_F,
    \qquad
    R_{\mathbb B}:=
    \frac{18M\sqrt n\,R_0}
         {\kappa\alpha\sqrt{\lambda_0}},
    \qquad
    \widehat P_t:=\frac{P_t}{R_{\mathbb B}}.
    \label{eq:empirical_displacement}
\end{equation}
The radius is $R_{\mathbb B}\approx7.88\times10^3$
and is shared across the width sweep.
For each run, we compute the minimum $K_t$,
maximum $S_t$, and maximum $\widehat P_t$ over
the recorded trajectory. Summary points and
error bars report the mean and one sample
standard deviation of these run-wise extrema
across the five student initializations, with
equal weight assigned to each run.

All 30 runs reach the target loss within the
practical horizon. The minimum normalized kernel
eigenvalue has mean $0.871\pm0.030$ at width 64
and $0.990\pm0.012$ at width 2000. The smallest
value recorded across all runs is $0.826$,
above the analytical reference $1/2$.
The maximum relative kernel change decreases
from $0.143\pm0.047$ at width 64 to
$0.031\pm0.007$ at width 2000
(Figure~\ref{fig:width_stability}).

The largest normalized displacement across all
runs is $1.67\times10^{-3}$, attained at width
2000 (Figure~\ref{fig:appendix_displacement}).
Thus, every observed trajectory remains far
inside the analytical neighborhood while
maintaining kernel positivity. The small
relative displacement and reduced kernel
variation at larger widths are consistent
with the stability mechanism used in the proof.

These observations concern repeated student
initializations on one fixed dataset and teacher.
They do not establish the necessity or tightness
of the sufficient width condition. In particular,
containment within the analytical radius at
these widths does not itself certify the uniform
kernel bounds, and the measured kernel diagnostics
do not verify every Jacobian-action and
linearization estimate used in the analysis.

\end{document}